\documentclass{article}

\usepackage[nonatbib,final]{icbinb_neurips_2020}

\usepackage[utf8]{inputenc} 
\usepackage[T1]{fontenc}    
\usepackage{hyperref}       
\usepackage{url}            
\usepackage{booktabs}       
\usepackage{amsfonts}       
\usepackage{nicefrac}       
\usepackage{microtype}      
\usepackage{amsmath, amssymb, amsthm}
\usepackage{bbm}
\usepackage[round]{natbib}
\usepackage{wrapfig}
\usepackage{algorithm}
\usepackage{algorithmic}
\usepackage{subcaption}
\usepackage{empheq}
\usepackage{xspace}

\newcommand*{\eg}{{\it e.g.}\@\xspace}

\newcommand{\E}{\mathbb{E}}
\DeclareMathOperator*{\argmin}{arg\,min}
\DeclareMathOperator*{\argmax}{arg\,max}

\DeclareMathOperator*{\trace}{tr}
\newcommand{\N}{\mathcal{N}}
\newcommand{\R}{\mathbb{R}}
\newcommand{\Meka}{\text{\textsc{Meka}}}
\newcommand{\sgd}{\text{\textsc{SGD}}}

\newcommand{\adam}{\text{\textsc{Adam}}}
\newcommand{\adadelta}{\text{\textsc{Adadelta}}}
\newcommand{\svrg}{\text{\textsc{SVRG}}}
\newcommand{\adagrad}{\text{\textsc{AdaGrad}}}
\newcommand{\adameka}{\textsc{AdaMeka}}

\newtheorem{proposition}{Proposition}
\newtheorem{lemma}{Lemma}

\usepackage{tikz}
\usetikzlibrary{arrows,shapes,plotmarks,decorations.pathmorphing}
\usetikzlibrary{backgrounds,calc,positioning,fadings}

\tikzset{every picture/.style={node distance=2cm}}
\tikzset{>=stealth'} 
\tikzstyle{graphnode} = 
   [circle,minimum size=22pt,text centered,text
     width=22pt,inner sep=0pt] 
\tikzstyle{var}   =[graphnode,fill=white]
\tikzstyle{obs}   =[graphnode,fill=black,text=white]

\tikzstyle{fac}   =[rectangle,fill=gray,minimum size=5pt]
\tikzstyle{facprior} =[rectangle,fill=black,text=white,minimum size=5pt]
\tikzstyle{edge}  =[draw=white,double=black,thick,-]
\tikzstyle{prior} =[rectangle, fill=black, minimum size=
5pt, inner sep=0pt]
\tikzstyle{dirprior} = [circle, fill=black, minimum
size=5pt, inner sep=0pt]

\definecolor{TUred}{RGB}{141,45,57}
\definecolor{TUdark}{RGB}{55,65,74}
\definecolor{TUgold}{RGB}{174,159,109}
\definecolor{TUlightgold}{RGB}{239,236,226}
\definecolor{TUgray}{RGB}{175,179,183}

\definecolor{TUsecondary1}{RGB}{65,90,140}   
\definecolor{TUsecondary2}{RGB}{0,105,170}   
\definecolor{TUsecondary3}{RGB}{80,170,200}  
\definecolor{TUsecondary4}{RGB}{130,185,160} 
\definecolor{TUsecondary5}{RGB}{125,165,75}  
\definecolor{TUsecondary6}{RGB}{50,110,30}   
\definecolor{TUsecondary7}{RGB}{200,80,60}   
\definecolor{TUsecondary8}{RGB}{175,110,150} 
\definecolor{TUsecondary9}{RGB}{180,160,150} 
\definecolor{TUsecondary10}{RGB}{215,180,105}
\definecolor{TUsecondary11}{RGB}{210,150,0}  
\definecolor{TUsecondary12}{RGB}{145,105,70} 

\newif\ifcomments

\commentstrue

\ifcomments
\newcommand{\ricky}[1]{\textcolor{blue}{[RC: #1]}}

\else
\newcommand{\ricky}[1]{}
\fi

\title{Self-Tuning Stochastic Optimization with Curvature-Aware Gradient Filtering}

\author{%
Ricky T. Q. Chen\thanks{Equal contribution.}\;
\thanks{University of Toronto. Vector Institute. \texttt{\{rtqichen, choidami, duvenaud\}@cs.toronto.edu}} \\
\And
Dami Choi\footnotemark[1]\;\;\footnotemark[2] \\
\And
Lukas Balles\footnotemark[1]\; 
\thanks{Max Planck Institute for Intelligent Systems, T\"ubingen, Germany. \texttt{\{lballes, ph\}@tue.mpg.de}} \\
\And
David Duvenaud\footnotemark[2] \\
\And
Philipp Hennig\footnotemark[3] \\
}

\begin{document}

\maketitle

\begin{abstract}

Standard first-order stochastic optimization algorithms base their updates solely on the average mini-batch gradient, and it has been shown that tracking additional quantities such as the curvature can help de-sensitize common hyperparameters.
Based on this intuition, we explore the use of exact per-sample Hessian-vector products and gradients to construct optimizers that are self-tuning and hyperparameter-free.
Based on a dynamics model of the gradient, we derive a process which leads to a curvature-corrected, noise-adaptive online gradient estimate.
The smoothness of our updates makes it more amenable to simple step size selection schemes, which we also base off of our estimates quantities.
We prove that our model-based procedure converges in the noisy quadratic setting. Though we do not see similar gains in deep learning tasks, we can match the performance of well-tuned optimizers and ultimately, this is an interesting step for constructing self-tuning optimizers.

\end{abstract}

\section{Introduction}\label{sec:intro}

Stochastic gradient-based optimization is plagued by the presence of numerous hyperparameters. 
While these can often be set to rule-of-thumb constants or manually-designed schedules, it is also common belief that a more information regarding the optimization landscape can help present alternative strategies such that manual tuning has less of an impact on the end result. 
For instance, the use of curvature information in the form of Hessian matrices or Fisher information can be used to de-sensitize or completely remove step size parameter~\citep{ypma1995historical,amari1998natural,martens2014new}, and the momentum coefficient can be set to reduce the local gradient variance~\citep{arnold2019reducing}.


\begin{wrapfigure}[15]{r}{0.45\linewidth}
\vspace{-1em}
    \centering
    \includegraphics[width=\linewidth]{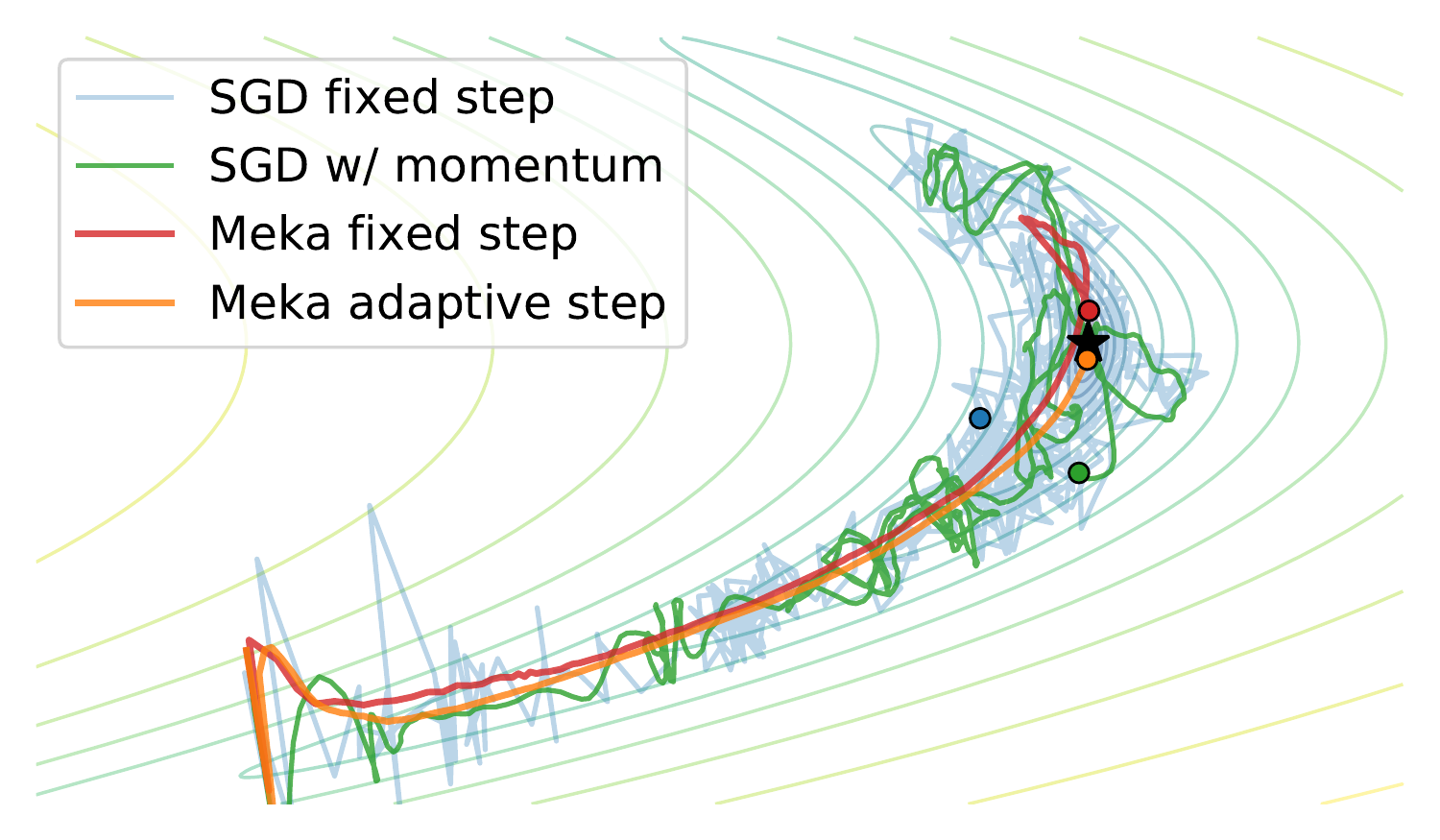}
    \caption{Stochastic gradient eventually goes into diffusion and does not converge. Our filtered gradients offer smooth convergence and complements adaptive step sizes.} 
    \label{fig:rosenbrock2d}
\end{wrapfigure}
Based on these intuitions, we investigate the use of efficient curvature and variance estimates during training to construct a \emph{self-tuning} optimization framework. 
Under a Bayesian paradigm, we treat the true gradient as the unobserved state of a dynamical system and seek to automatically infer the true gradient conditioned on the history of parameter updates and stochastic gradient observations.

Our method is enabled by evaluations of exact \emph{per-sample} gradients and Hessian-vector products. With recent improvements in automatic differentiation tooling~\citep[\eg ,][]{jax2018github,agarwal2019auto,dangel2020backpack}, this matches the asymptotic time cost of minibatch gradient and Hessian-vector product evaluations.

While our framework contains the good properties of both curvature-based updates and variance reduction---which are attested in toy and synthetic scenarios---we do not observe significant improvements empirically in optimizing deep neural networks. Notably, our approach can be viewed as an explicit form of the implicit gradient transport of \citet{arnold2019reducing}, yet it does not achieve the same acceleration empirically observed in practice. While we do not fully understand this behavior, we analyze the estimated quantites along the training trajectory and hypothesize that our method has a higher tendency of going down high-variance high-curvature regions whereas standard stochastic gradient descent is repelled from such regions due to gradient variance. This potentially serves as a downside of our method in the deep learning setting. Regardless, the use of efficient variance estimation and interpretation of gradient estimation as Bayesian filtering are useful constructs in the development of self-tuning stochastic optimization.


\section{Bayesian Filtering for Stochastic Gradients}\label{sec:gradfilter}

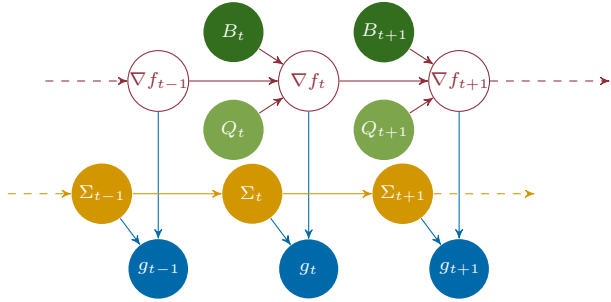
\begin{figure}
\centering
\begin{tikzpicture}\scriptsize
    \begin{scope}[draw=TUred,text=TUred]
    \node[var,draw=TUred] (gm) at (0,0) {$\nabla f_{t-1}$};
    \node[var,draw=TUred] (g0) at (2,0) {$\nabla f_{t}$} edge[<-] (gm);
    \node[var,draw=TUred] (gp) at (4,0) {$\nabla f_{t+1}$} edge[<-] (g0);

    \draw (-1.5,0) edge[->,dashed] (gm);
    \draw (gp) edge[->,dashed] (6,0);
    \end{scope}

    \node[obs,TUsecondary2,text=white] (ym) at (0,-2.5) {$g_{t-1}$} edge[<-,draw=TUsecondary2] (gm);
    \node[obs,TUsecondary2,text=white] (y0) at (2,-2.5) {$g_{t}$} edge[<-,draw=TUsecondary2] (g0);
    \node[obs,TUsecondary2,text=white] (yp) at (4,-2.5) {$g_{t+1}$} edge[<-,draw=TUsecondary2] (gp);


    \node[obs,TUsecondary11,text=white] (sm) at (-0.75,-1.5) {$\Sigma_{t-1}$} edge[->,TUsecondary2] (ym);
    \node[obs,TUsecondary11,text=white] (s0) at (1.25,-1.5) {$\Sigma_{t}$} edge[->,TUsecondary2] (y0);
    \node[obs,TUsecondary11,text=white] (sp) at (3.25,-1.5) {$\Sigma_{t+1}$} edge[->,TUsecondary2] (yp);

    \draw (-2,-1.5) edge[->,TUsecondary11,dashed] (sm);
    \draw (sm) edge[->,TUsecondary11] (s0);
    \draw (s0) edge[->,TUsecondary11] (sp);
    \draw (sp) edge[->,TUsecondary11,dashed] (5,-1.5);

    \node[obs,TUsecondary6,text=white] (Bm) at (1,0.65) {$B_{t}$} edge[->,TUred] (g0); 
    \node[obs,TUsecondary6,text=white] (B0) at (3,0.65) {$B_{t+1}$} edge[->,TUred] (gp);


    \node[obs,TUsecondary5,text=white] (Bm) at (1,-0.65) {$Q_{t}$} edge[->,TUred] (g0); 
    \node[obs,TUsecondary5,text=white] (B0) at (3,-0.65) {$Q_{t+1}$} edge[->,TUred] (gp); 

\end{tikzpicture}
\caption{Graphical model of the hidden Markov dynamics model.
The main idea of our algorithm is that the dynamics parameters can be cheaply estimated on each minibatch, and smoothed across time using exact Kalman filter inference.
These dynamics parameters are the gradient variance $\Sigma$, the directional curvature $B\delta$ and its variance $Q$.
We stabilize $\Sigma$ with an exponential moving average, which is effectively another, more elementary form of Kalman filtering.}
\label{fig:graph}
\end{figure}

We consider stochastic optimization problems of the general form
\begin{equation}
    \arg \min_{\theta \in \R^d} f(\theta),\quad f(\theta) = \E_{\xi}\left[\tilde{f} (\theta, \xi)\right]
\end{equation}
where we only have access to samples $\xi$.
Stochastic gradient descent---the prototypical algorithm for this setting---iteratively updates $\theta_{t+1} = \theta_t - \alpha_t g_t$, where
\begin{equation}
    \label{eq:minibatch_stochastic_gradient}
    g_t = \frac{1}{n} \sum_{i=1}^n \nabla_\theta \tilde{f}(\theta_t, \xi_t^{(i)}), \quad \xi_t^{(1)}, \dotsc, \xi_t^{(n)} \overset{\text{iid}}{\sim} p(\xi),
\end{equation}
and $\alpha_t$ is a scalar step size.
We may use notational shorthands like $f_t = f(\theta_t)$, $\nabla f_t = \nabla f(\theta_t)$.

SGD is hampered by the effects of gradient noise.
It famously needs a decreasing step size schedule to converge; used with a constant step size, it goes into diffusion in a region around the optimum~\citep[see, \eg ,][]{bottou2018optimization}.
Gradient noise also makes stochastic optimization algorithms difficult to tune.
In particular, unreliable directions are not amenable to step size adaptation.

To stabilize update directions, we build a framework for estimating the true gradient $\nabla f$ based on Kalman filtering. This can also be viewed as a variance reduction method, but does not require the typical finite-sum structure assumption of \eg \citet{schmidt2017minimizing,johnson2013accelerating}.




\subsection{Dynamical System Model}

We treat the true gradient $\nabla f_t$ as the latent state of a dynamical system. This dynamical system is comprised of an observation model $p(g_t \;|\; f_t )$ and a dynamics model $p(\nabla f_t \;|\; \nabla f_{t-1}, \delta_{t-1})$ where $\delta_{t-1} = \theta_{t} - \theta_{t-1}$ is the update direction.
We will later choose $\delta_t$ to be depend on our variance-reduced gradient estimates, but the gradient inference framework itself is agnostic to the choice of $\delta_t$.

\paragraph{The observation model}
$p(g_t \;|\; \nabla f_t)$ describes how the gradient observations relate to the state of the dynamical system.
In our case, it is relatively straight-forward, since $g_t$ is simply an unbiased stochastic estimate of $\nabla f_t$, but the exact distribution remains to be specified.
We make the assumption that $g_t$ follows a Gaussian distribution,
\begin{equation}
    \label{eq:observation_model}
    g_t \;|\; \nabla f_t \sim \mathcal{N}(\nabla f_t, \Sigma_t),
\end{equation}
with covariance $\Sigma_t$.
Since $g_t$ is the mean of iid terms (Eq.~\ref{eq:minibatch_stochastic_gradient}), this assumption is supported by the central limit theorem when sufficiently large batch sizes are used. 

\paragraph{The dynamics model}
$p(\nabla f_t \;|\; \nabla f_{t-1})$ describes how the gradient evolves between iterations.
We base our dynamics model on a first order Taylor expansion of the gradient function \textbf{centered at $\mathit{\boldsymbol\theta_{\mathbf{t}}}$}, $\nabla f(\theta_{t-1}) \approx \nabla f(\theta_t) - \nabla^2 f(\theta_t) \delta_{t-1}$.
We propose to approximate the gradient dynamics by computing a stochastic estimate of the Hessian-vector product, $B_t\delta_{t-1}$, where $\E [B_t] = \nabla^2 f(\theta_t)$.
Again, we make a Gaussian noise assumption.
This implies the dynamics model
\begin{equation}
    \label{eq:dynamics_model}
    \nabla f_t \;|\; \nabla f_{t-1} \sim \mathcal{N}(\nabla f_{t-1} + B_t \delta_{t-1}, Q_{t}).
\end{equation}
where $Q_t$ is the covariance of $B_t\delta_{t-1}$, taking into account the stochasticity in $B_t$. 

A key insight is that the parameters $B_t\delta_{t-1}, Q_t, \Sigma_t$ of the model can all be ``observed'' directly using automatic differentiation of the loss on each minibatch of samples.
We use the Hessian at $\theta_t$ so that the Hessian-vector product can be simultaneously computed with $g_t$ with just one extra call to automatic differentiation (or ``backward pass'') in each iteration (note this does not require constructing the full matrix $B_t$).
The variances $Q_t$ and $\Sigma_t$ can also be empirically estimated with some memory overhead by using auto-vectorized automatic differentiation routines.
We discuss implementation details later in Section~\ref{sec:practical_implementation}.

\subsection{Filtering Framework for Gradient Inference}
As Equations~\eqref{eq:observation_model} and~\eqref{eq:dynamics_model} define a linear-Gaussian dynamical system, exact inference on the true gradient conditioned on the history of gradient observations $p(\nabla f_t | g_{1:t}, \delta_{1:{t-1}})$ takes the form of the well-known Kalman filtering equations~\citep{kalman1960} \citep[review in][]{sarkka2013bayesian}: We define parameters $m_t^-$, $m_t$, $P_t^-$ and $P_t$ such that
\begin{equation}\label{eq:grad_posterior}
\begin{split}
\nabla f_t \mid g_{1:t-1}, \delta_{1:{t-1}} &\sim \N ( m_t^-,\; P_t^- ) \\ 
\nabla f_t \mid g_{1:t}, \delta_{1:{t-1}} &\sim \N ( m_t,\; P_t ).
\end{split}
\end{equation}
Starting from a prior belief $\nabla f_0 \sim \mathcal{N}(m_0, P_0)$, these parameters are updated iteratively:
\begin{align}
    m_t^- &= m_{t-1} + B_t\delta_{t-1}, & P_t^- &= P_{t-1} + Q_{t-1} \label{eq:kalman_predict} \\
    K_t &= P_t^- ( P_t^- + \Sigma_t )^{-1} \label{eq:kalman_gain} &  & \\
    m_t &= (I - K_t)m_t^- + K_t g_t, & P_t &= (I - K_t)P_t^-(I - K_t)^T + K_t \Sigma_t K_t^T \label{eq:kalman_correct}
\end{align}

Equation \eqref{eq:kalman_predict} is referred to as the \emph{prediction} step as it computes mean and covariance of the predictive distribution $p(\nabla f_t \vert g_{1:t-1})$.
In our setting, it predicts the gradient $\nabla f_t$ based on our estimate of the previous gradient ($m_{t-1}$) and the Hessian-vector product approximating the change in gradient from the step $\theta_t = \theta_{t-1} + \delta_{t-1}$.
Equation~\eqref{eq:kalman_correct} is the \emph{correction} step.
Here, the local stochastic gradient evaluation $g_t$ is used to correct the prediction.
Importantly, the \emph{Kalman gain}~\eqref{eq:kalman_gain} determines the blend between the prediction and the observations according to the uncertainty in each.

The resulting algorithm gives an online estimation of the true gradients as the parameters $\theta_t$ are updated. 
We refer to this framework as \Meka{}, loosely based on \emph{model-based Kalman-adjusted gradient estimation}.
During optimization, we may use the posterior mean $m_t$ as a variance-reduced gradient estimator and take steps in the direction of $\delta_t = - \alpha_t m_t$.

We note two key insights enabling \Meka{}: First, all parameters of the filter are not set \emph{ad hoc}, but are directly evaluated or estimated using automatic differentiation.
Secondly, the dynamics model makes explicit use of the Hessian to predict gradients. This is a first-order update. In contrast to second-order methods, like quasi-Newton methods, \Meka{} does not try to estimate the Hessian from gradients, but instead leverages a (noisy) projection with the actual Hessian to improve gradient estimates. This is both cheaper and more robust than second-order methods, because it does not involve solving a linear system.


\subsection{\adam-style Update Directions}

While $\Meka$ produces variance-reduced gradient estimates, it does not help with ill-conditioned optimization problems, a case where full batch gradient descent can perform poorly. To alleviate this, we may instead take update directions motivated by the \adagrad~\citep{duchi2011adaptive} line of optimizers. We follow \adam~\citep{kingma2014adam} which proposes dividing the first moment of the gradient element-wise by the square root of the second moment, to arrive at
\begin{equation}
    \delta_t = \alpha_t \frac{m_t}{\sqrt{m_t + \text{diag}(P_t)} + \varepsilon}.
\end{equation}
where $\varepsilon$ is taken for numerical stability and simply set to $10^{-8}$.
Whereas $\adam$ makes use of two exponential moving averages to estimate the first and second moments of $g_t$, we have estimates automatically inferred through the filtering framework. We refer to this variant as \adameka.

\section{Uncertainty-informed Step Size Selection}

We can adopt a similar Bayesian filtering framework for probabilistic step size adaptation. Our step size adaptation will be a simple enhancement to the quadratic rule, but takes into account uncertainty in the stochastic regime and is much more robust to stochastic observations. The standard quadratic rule if the objective $f$ can be computed exactly is $\alpha_{\textnormal{quadratic}} := \frac{-\delta_t^T \nabla f_t}{\delta_t^T \nabla^2 f_{t-1} \delta_t}$,
which is based on minimizing a local quadratic approximation $f(\theta_{t} + \alpha_t \delta_t) - f_t \approx \alpha\delta_t^T \nabla f_t + \frac{\alpha^2}{2} \delta_t^T \nabla^2 f_{t-1} \delta_t$.

However, since we only have access to stochastic estimates of $\nabla f$ and $\nabla^2 f$, na\"ively taking this step size with high variance samples results in unpredictable behavior and can cause divergence during optimization. 
To compensate for the stochasticity and inaccuracy of a quadratic approximation, adaptive step size approaches often include a ``damping'' term~(\eg \citet{martens2010hfopt})---where a constant is added to the denominator---and an additional scaling factor on $\alpha_t$, both of which aim to avoid large steps but introduces more hyperparameters.

As an alternative, we propose a scheme that uses the variance of the estimates to adapt the step size, only taking steps into regions where we are confident about minimizing the objective function. Once again leveraging the availability of $Q_t$ and $\Sigma_t$, our approach allows automatic trade-off between minimizing a local quadratic approximation and the uncertainty over large step sizes, foregoing manual tuning methods such as damping.

We adopt a similar linear-Gaussian dynamics model for tracking the true objective $f_t$, with the same assumptions as in Section~\ref{sec:gradfilter}. Due to its similarity with Section~\ref{sec:gradfilter}, we delegate the derivations to Appendix~\ref{app:func}. We again define the posterior distribution,
\begin{equation}
    f_t \mid y_{1:t}, \delta_{1:t-1} \sim \mathcal{N}( u_t, s_t ).
\end{equation}
where $u_t$ and $s_t$ are inferred using the Kalman update equations. Finally, setting $f_{t+1} = f(\theta_t + \alpha_t\delta_t)$ for some direction $\delta_t$, we have a predictive model of the change in function value as
\begin{equation}\label{eq:df_dist}
f_{t+1} - f_t \mid y_{1:t}, g_{1:t}, \delta_{1:t} \sim \mathcal{N}\bigg( \alpha_t \delta_t^Tm_t + \frac{\alpha_t^2}{2} \delta_t^TB_t\delta_t, 2s_t + \alpha_t^2 \delta_t^T P_t\delta_t + \frac{\alpha_t^4}{4} \delta_t^T Q_t\delta_t \bigg)
\end{equation}
Contrasting this with the simple quadratic approximation, the main difference is now we take into account the uncertainty in $f_t$, $\nabla f_t$, and $\nabla^2 f_t$.
Each term makes different contributions to the variance as $\alpha_t$ increases, corresponding to different trade-offs between staying near where we are more certain about the function value and exploring regions we believe have a lower function value.
Explicitly specifying this trade-off gives an \emph{acquisition function}.
These decision rules are typically used in the context of Bayesian optimization~\citep{boreview}, but we adopt their use for step size selection.

\subsection{Acquisition Functions for Step Size Selection}

Computing the optimal step size in the context of a long but finite sequence of optimization steps is intractable in general, but many reasonable heuristics have been developed.
These heuristics usually balance immediate progress against information gathering likely to be useful for later steps.

One natural and hyperparameter-free heuristic is maximizing the \emph{probability of improvement} (PI)~\citep{kushner1964new},
\begin{equation}\label{eq:pi}
    \alpha_\textnormal{PI} := \argmax_\alpha \mathbb{P}\left( f_{t+1} - f_t \leq 0 \mid y_{1:t}, g_{1:t} \right)
\end{equation}
which is simply the cumulative distribution function of \eqref{eq:df_dist} evaluated at zero. 

\begin{figure}
    \centering
    \begin{subfigure}[b]{0.35\linewidth}
        \includegraphics[width=\linewidth]{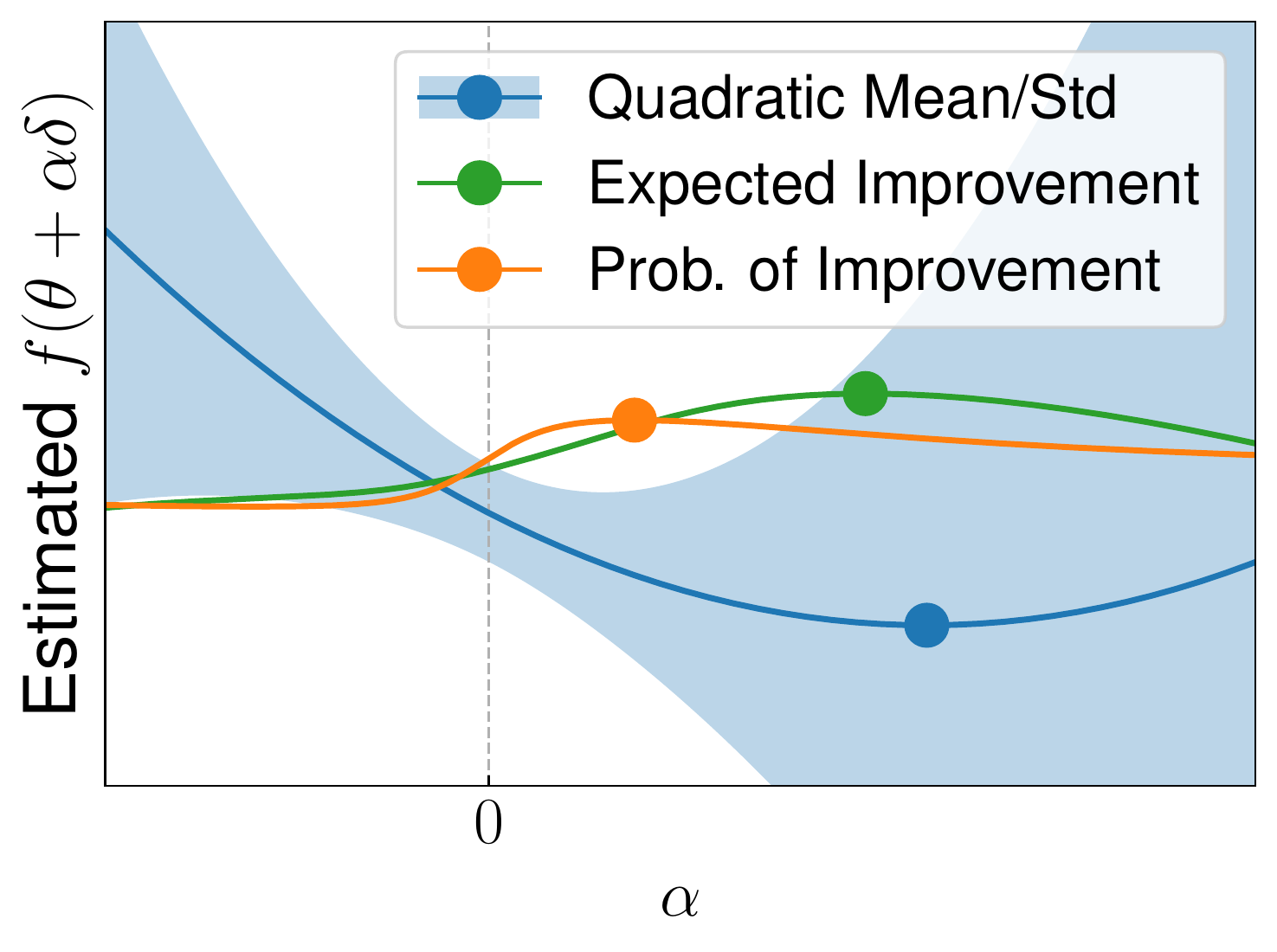}
        \caption{Positive curvature}
        \label{fig:negcurv_nolambda}
    \end{subfigure}
    \begin{subfigure}[b]{0.35\linewidth}
        \includegraphics[width=\linewidth]{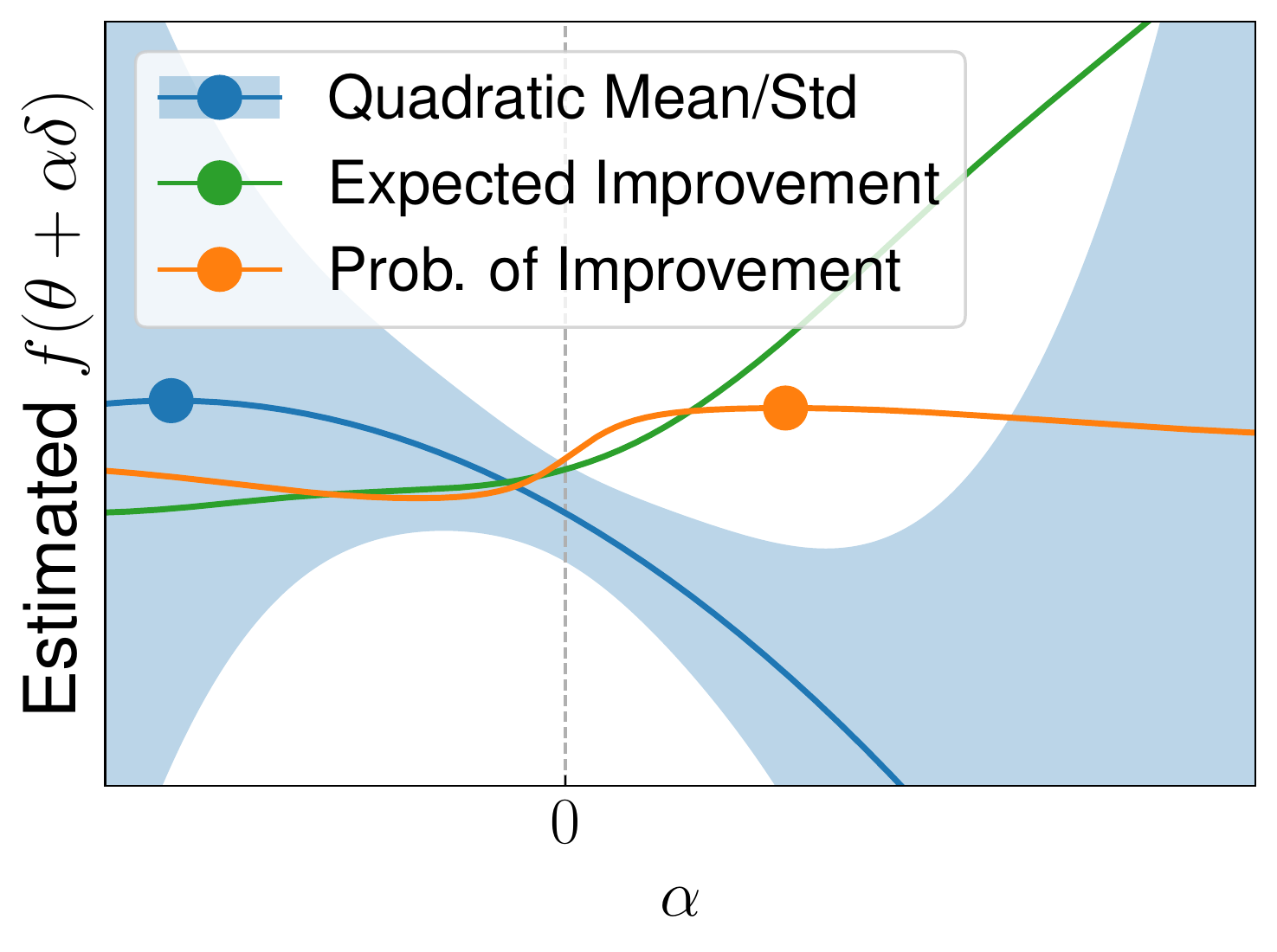}
        \caption{Negative curvature}
        \label{fig:negcurv_lambda}
    \end{subfigure}
    \caption{Illustration of different acquisition functions for selecting a step size $\alpha$, based on the mean and variance of our local quadratic estimate of the loss surface.}
    \label{fig:acquisition_fns}
\end{figure}

Figure~\ref{fig:acquisition_fns} visualizes the different step sizes chosen by maximizing different acquisition functions.
The heuristic of choosing the minimum of the quadratic approximation can be a poor decision when the uncertainty rises quickly. The optimum for PI interpolates between zero and the quadratic minimum in such a way that avoids regions of high uncertainty.
Expected improvement~\citep{jones1998efficient} is another popular acquisition function; however, in tests we found it to not be as robust as PI and often results in step sizes that require additional scaling.

Maximizing probability of improvement is equivalent to the following optimization problem
\begin{equation}\label{eq:pi_loss}
    \alpha_{\text{PI}} = \argmin_\alpha \frac{-\alpha \delta_t^T m_t + \frac{\alpha^2}{2} \delta_t^TB_t\delta_t}{\sqrt{ 2s_t + \alpha^2\delta_t^TP_t\delta_t + \frac{\alpha^4}{4} \delta_t^TQ_t\delta_t }}.
\end{equation}
We numerically solve for $\alpha_{\text{PI}}$ using Newton's method, which itself is only a small overhead since we only optimize in one variable with fixed constants: no further evaluations of $f$ are required.
We also note that there is exactly one optimum for $\alpha \in \R^+$.
For optimization problems where negative curvature is a significant concern, we include a third-order correction term that  ensures finite and positive step sizes~(details in Appendix~\ref{app:negcurv_lambda}).

\section{A Practical Implementation}
\label{sec:practical_implementation}

While the above derivations have principled motivations and are free of hyperparameters, a practical implementation of \Meka{} is not entirely straightforward. Below we discuss some technical aspects, simplifications and design choices that increase stability in practice, as well as recent software advances that simplify the computation of quantities of interest. 

\paragraph{Computing Per-Example Quantities for Estimating Variance}
Recent extensions for automatic differentiation in the machine learning software stack~\citep{jax2018github,agarwal2019auto,dangel2020backpack} implement an automatic vectorization \texttt{map} function.
Vectorizing over minibatch elements allows efficient computation of gradients and Hessian-vector products of neural network parameters with respect to each data sample independently. 
These advances allow efficient computation of the empirical variances of gradients and Hessian-vector products, and enable our filtering-based approach to gradient estimation. 


\paragraph{Stabilizing Filter Estimates}

Instead of working with the full covariance matrices $\Sigma_t$ and $Q_t$, we approximate them as scalar objects $\sigma_tI$ and $q_tI$, with $\sigma_t,q_t\in\mathbb{R}_+$ by averaging over all dimensions. We have experimented with diagonal matrices, but found that the scalar form increases stability, generally performing better on our benchmarks. 
Furthermore, we use an exponential moving average for smoothing the estimated gradient variance $\sigma_t$ as well as the adaptive step sizes $\alpha_t$. 
The coefficients of these exponential moving average are kept at $0.999$ in our experiments and seem to be quite insensitive, with values in $\{0.9, 0.99, 0.999\}$ all performing near identically (see Appendix~\ref{app:beta_test}). 





\section{Related Work}

Designing algorithms that can self-tune its own parameters is a central theme in optimization~\citep{eiben2011parameter,yang2013framework}; we focus on the stochastic setting, building on and merging ideas from several research directions.
The Bayesian filtering framework itself has previously been applied to stochastic optimization.
To the best of our knowledge, the idea goes back to \citet{bittner2004kalman} who used a filtering approach to devise an automatic stopping criterion for stochastic gradient methods.
\citet{Patel_2016} proposed filtering-based optimization methods for large-scale linear regression problems.
\citet{vuckovic2018kalman} and \citet{mahsereci2018probabilistic} used Kalman filters on general stochastic optimization problems with the goal of reducing the variance of gradient estimates.
In contrast to our work, none of these existing approaches leverage evaluations of Hessian-vector products to give curvature-informed dynamics for the gradient.

In terms of online variance reduction, \citet{gower2017tracking} have discussed the use of Hessian-vector products to correct the gradient estimate; however, they propose methods that approximate the Hessian whereas we compute exact Hessian-vector products by automatic differentiation.
\citet{arnold2019transporting} recently proposed an implicit gradient transport formula analogous to our dynamics model, but they require a rather strong assumption that the Hessian is the same for all samples and parameter values. In contrast, we focus on explicitly transporting via the full Hessian. This allows us to stay within the filtering framework and automatically infer the gain parameter, whereas the implicit formulation of \citet{arnold2019transporting} requires the use of a manually-tuned averaging schedule.

Step size selection under noisy observations is a difficult problem and has been tackled from multiple viewpoints. Methods include meta-learning approaches~\citep{almeida1999parameter,schraudolph1999local,plagianakos2001learning,yu2006fast,baydin2017online} or by assuming the interpolation regime~\citep{vaswani2019painless,berrada2019training}. \citet{rolinek2018l4} proposed extending a linear approximation to adapt step sizes but introduces multiple hyperparameters to adjust for the presence of noise, whereas we extend a quadratic approximation and automatically infer parameters based on noise estimates. Taking into account observation noise, \citet{mahsereci2017probabilistic} proposed a probabilistic line search that is done by fitting a Gaussian process to the optimization landscape. However, inference in Gaussian processes is more costly than our filtering approach.

\section{Convergence in the Noisy Quadratic Setting}


%
\begin{wrapfigure}[17]{r}{0.4\linewidth}
\vspace{-1em}
    \centering
    \includegraphics[width=\linewidth]{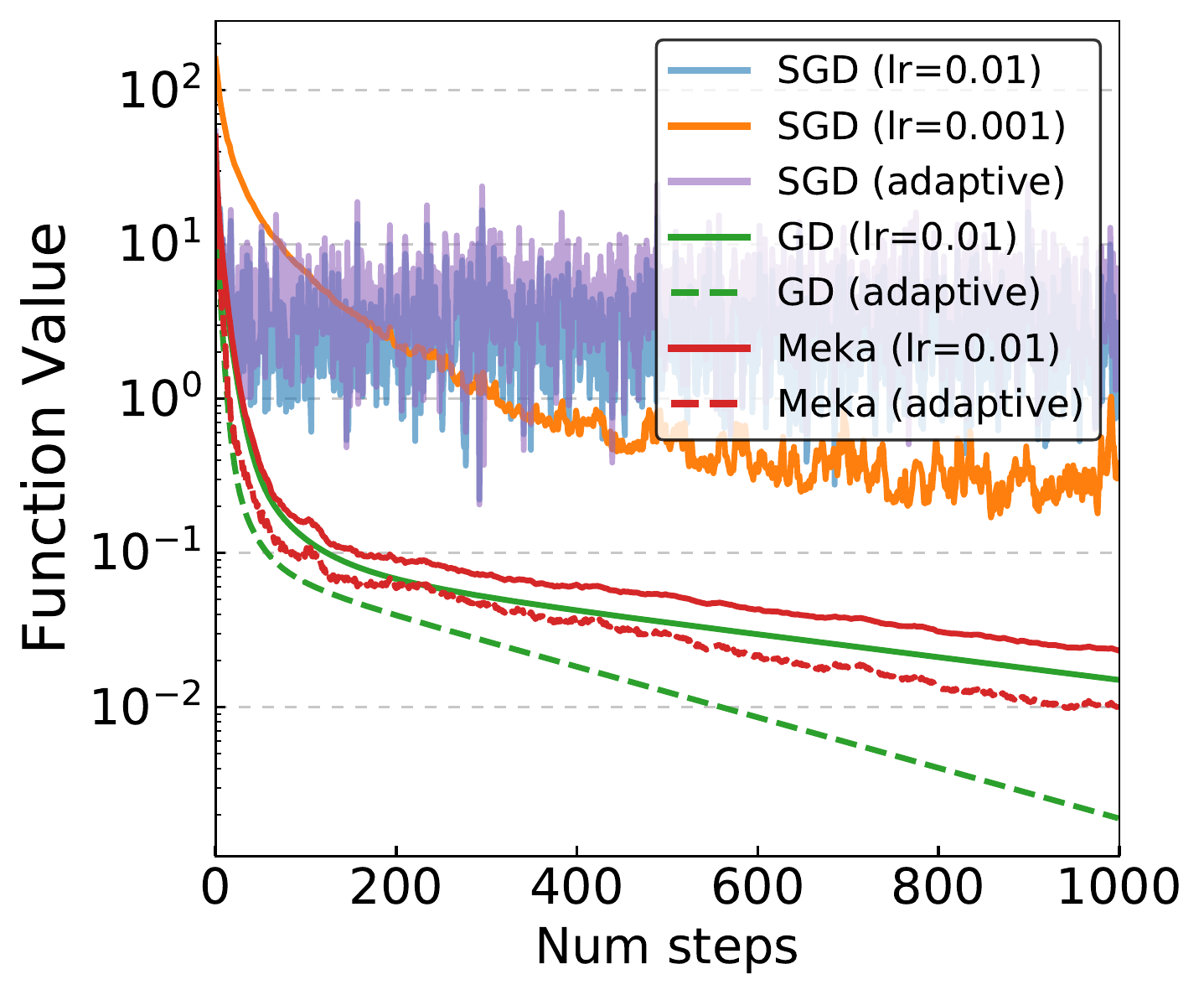}
    \caption{Filtered gradients converge with a fixed step size in the noisy quadratic regime, whereas \sgd{} results in diffusion for the same step size.}
    \label{fig:noisy_quadratic}
\end{wrapfigure}
As a motivating example, consider a simple toy problem, where
\begin{equation}
    \label{eq:quadratic_toy_problem}
    f(\theta, \xi) = \frac{1}{2} (\theta - \xi)^T H (\theta - \xi),
\end{equation}
i.e., a mixture of quadratic functions with identical Hessian but varying location determined by the ``data'' $\xi$.
The full gradient is $\nabla f(\theta) = H(\theta - \E[\xi])$ and per-example gradients evaluate to $\nabla f(\theta, \xi) = H(\theta - \xi) = \nabla f(\theta) - H(\xi - \E[\xi]$).
Hence, we have additive gradient noise with covariance $\Sigma = H\mathbf{Cov}[\xi] H^T$ independent of $\theta$.
Moreover, since the Hessian $\nabla^2 f(\theta, \xi) = H$ is independent of $\xi$, we have 
that $B_t\delta_{t-1}\equiv \nabla f_t - \nabla f_{t-1}$.
The covariance $Q_t$ is zero and the filter equations simplify to
\begin{equation}\label{eq:quadratic_toy_problem_filter_equations}
\begin{split}
    & K_t = P_{t-1} (P_{t-1}+ \Sigma)^{-1},\\ 
    & m_t = (I - K_t) (m_{t-1} + B_t\delta_{t-1}) + K_t g_t,\\ 
    & P_t = (I - K_t) P_{t-1},
\end{split}
\end{equation}
initialized with $m_0=g_0$, $P_0=\Sigma$.
The filter covariance $P_t$ contracts in every step and in fact, shrinks at a rate of $O(1/t)$, meaning that the filter will narrow in on the exact gradient.
We show that this enables $O(1/t)$ convergence with a \emph{constant} step size.
\begin{proposition}
\label{proposition:quadratic_toy_problem_convergence}
Assume a problem of the form \eqref{eq:quadratic_toy_problem} with  $\mu I \preceq H \preceq LI$.
If we update $\theta_{t+1} = \theta_t - \alpha m_t$ with $\alpha\leq 1/L$ and $m_t$ obtained via Eq.~\eqref{eq:quadratic_toy_problem_filter_equations}, then $\E[f(\theta_t) - f_\ast] \in O\left(1/t \right)$.
\end{proposition}
Figure~\ref{fig:noisy_quadratic} shows experimental results for such a noisy quadratic problem of dimension $d=20$ with a randomly-generated Hessian (with condition number $> 1000$) and $\xi\sim\mathcal{N}(0, I)$.
Using \sgd{} with a high learning rate simply results in diffusion, and setting the learning rate smaller results in slow convergence. Gradient descent (GD) converges nicely with the high learning rate, and using adaptive steps sizes leads to a better convergence rate. The filtered gradients from \Meka{} converge almost as well as gradient descent, and adaptive step sizes provide an improvement. On the other hand, \sgd{} produces unreliable gradient directions and does not work well with adaptive step sizes. We note that the stochastic gradient has a full covariance matrix and does not match our modeling assumptions, as our model uses a diagonal covariance for efficiency. Even so, the training loss of \Meka{} follows that of gradient descent very closely after just a few iterations.

\section{Classification Experiments}

Next we test and diagnose our approach on classification benchmarks, MNIST and CIFAR-10. We use JAX's~\citep{jax2018github} vectorized map functionality for efficient per-example gradients and Hessian-vector products.
For MNIST, we test using a multi-layer perceptron (MLP); for CIFAR-10, a convolutional neural network (CNN) and a residual network (ResNet-32)~\citep{he2016deep,he2016identity}. 
One key distinction is we replace the batch normalization layers with group normalization~\citep{wu2018group} as batch-dependent transformations conflict with our assumption that the gradient samples are independent. We note that the empirical per-iteration cost of \Meka{} is $1.0$--$1.6 \times$ that of \sgd{} due to the computation of Hessian-vector products.
Full experiment details are provided in Appendix~\ref{app:experiment_details}. A detailed comparison to tuned baseline optimizers is presented in Appendix~\ref{app:comparison_with_all}.

\paragraph{Online Variance Reduction}

\begin{wrapfigure}[15]{r}{0.4\linewidth}
\vspace{-1em}
\centering
    \includegraphics[width=\linewidth]{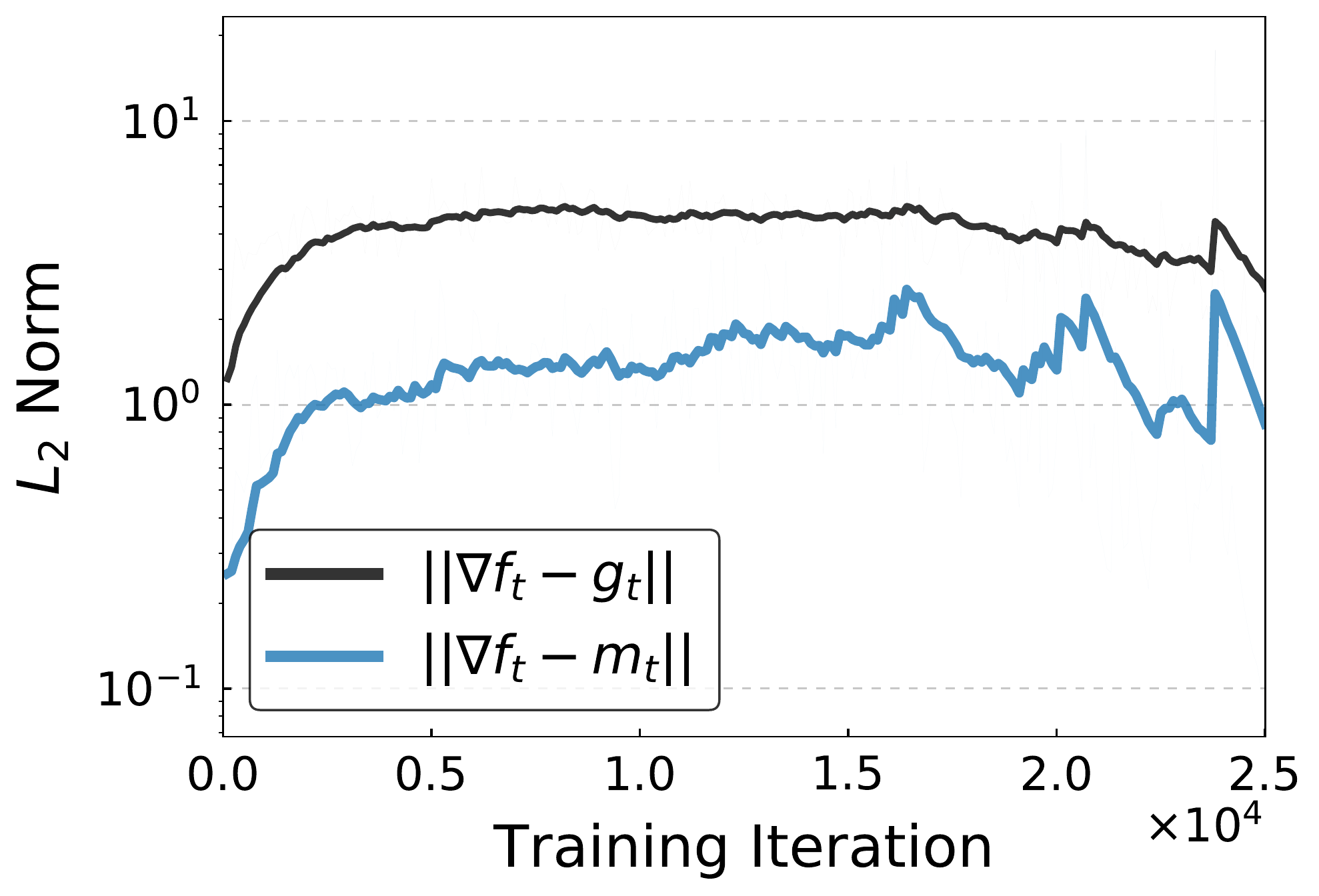}
    \caption{\Meka{}'s estimated gradients are closer to the true full-batch gradient in $L_2$ norm than stochastically observed gradients by a factor of around $5$.}
    \label{fig:l2_norm}
\end{wrapfigure}
We test whether the filtering procedure is correctly aligning the gradient estimate with the true gradient. For this, we use CIFAR-10 with a CNN and no data augmentation, so that the true full-batch gradient over the entire dataset can be computed. 
Figure~\ref{fig:l2_norm} shows the $L_2$ norm difference between the gradient estimators and the full-batch gradient $\nabla f_t$. \Meka{}'s estimated gradients are closer to the true around by around a factor of 5 compared to the minibatch gradient sample.



\begin{figure}
    \centering
    \includegraphics[width=0.82\linewidth]{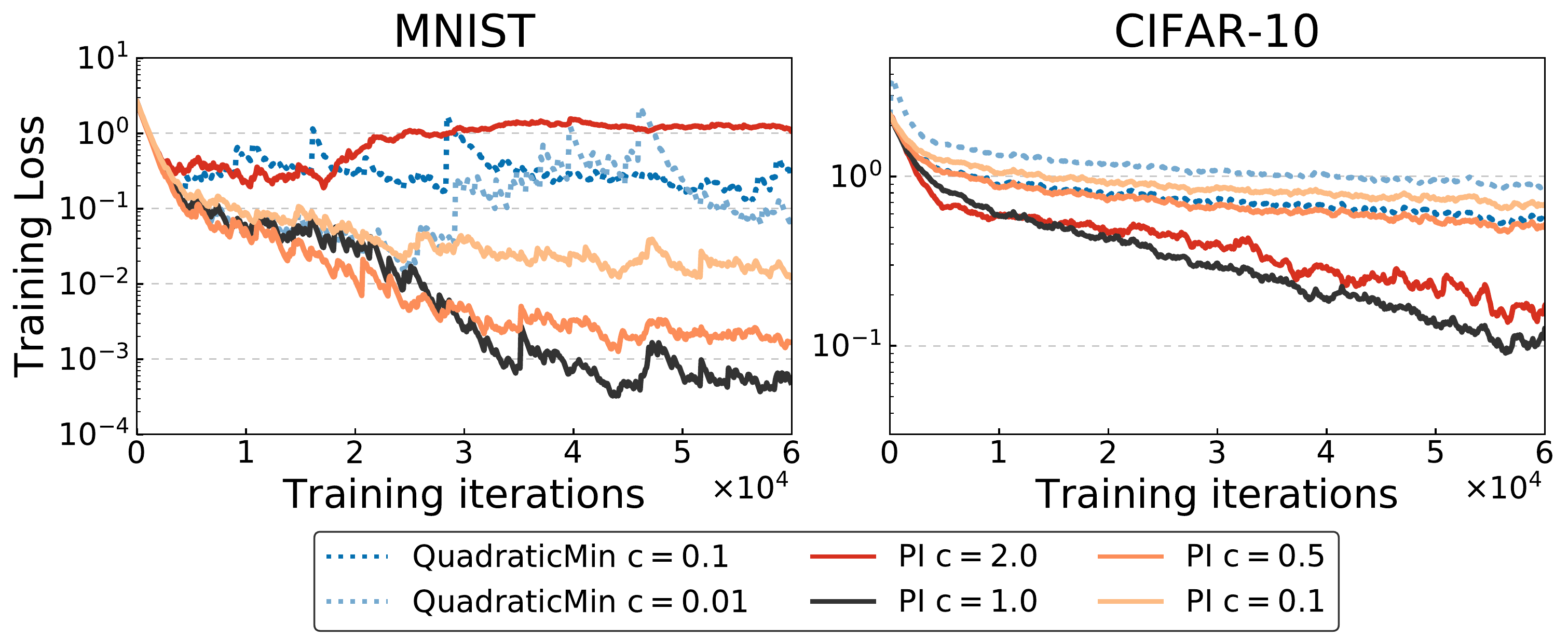}
    \caption{Adaptive step sizes based on probability of improvement work best without any additional scaling factor $c$ for modifying the update rule: $\theta_{t+1} = \theta_t + c \alpha_t \delta_t$. }
    \label{fig:adaptive_comparison}
\end{figure}

\paragraph{Adaptive Step Sizes are Appropriately Scaled}
\label{sec:adaptive_scaling}

Without uncertainty quantification, the quadratic minimum step size scheme tends to result in step sizes too large.
As such, one may include a scaling factor such that the update is modified as $\theta_{t+1} = \theta_t - c \alpha_t \delta_t$. In contrast, we find that the adaptive step sizes based on probability of improvement (PI) are already correctly scaled in the sense that a $c$ different from $1.0$ will generally result in worse performance.
Figure~\ref{fig:adaptive_comparison} shows a comparison of different values for $c$ for the quadratic and PI~\eqref{eq:pi} adaptive schemes.
We plot expected improvement in Appendix~\ref{app:adaptive_comparison_full}, which performs poorly and requires non-unit scaling factors.


\subsection{Adaptive Step Sizes Dives into High-curvature, High-variance Regions}
\label{app:mekadive}

\begin{figure}
    \centering
    \includegraphics[width=\linewidth]{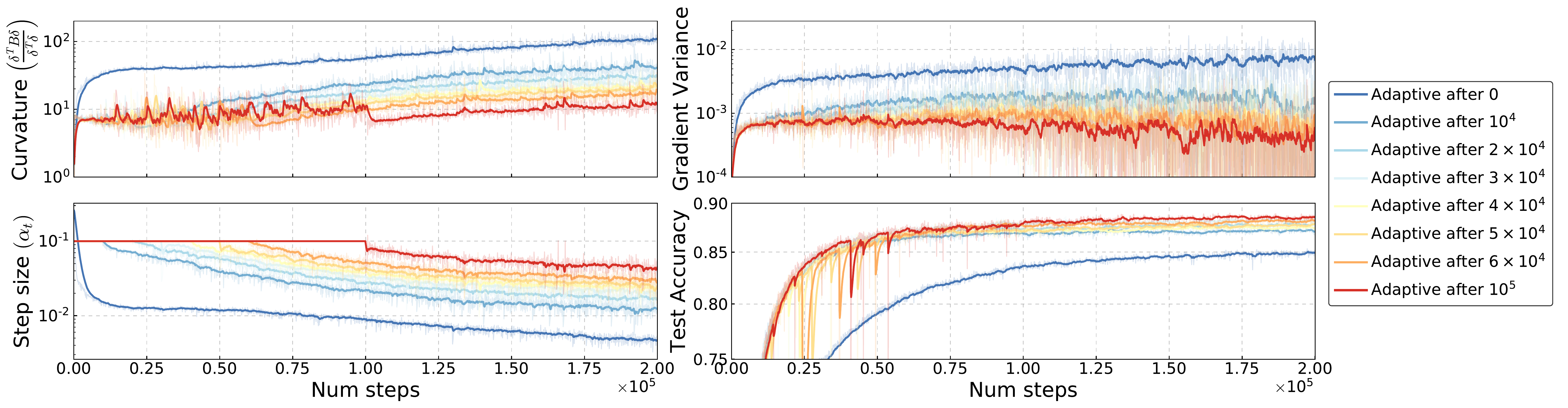}
    \caption{The performance of \Meka{} with adaptive step sizes on ResNet-32 can be explained by quantities captured during optimization. \Meka{} reaches high-curvature high-variance local minima, as soon as adaptive step sizes are used.}
    \label{fig:adaptdelay}
\end{figure}

A core aspect of our filtering approach is the ability to estimate quantities of interest during optimization. We now use these to help understand the loss landscape of ResNet-32 on CIFAR-10. 
We find that a cause for slow convergence of \Meka{} with adaptive step sizes is due to an abundance of minima that are usually too high variance for standard \sgd.

Figure~\ref{fig:adaptdelay} shows estimates of the normalized curvature along the descent direction $\smash{\frac{\delta^TB_t\delta}{\delta^T\delta}}$ as well as the per-sample gradient variance, averaged over parameters.
To understand the loss landscape along the trajectory of optimization, we use multiple runs of \Meka{} with the same initialization. Each run takes a different fixed number of constant-size steps before switching to the adaptive step size scheme. 

It is clear that immediately after switching to adaptive step sizes, \Meka{} falls into an increasingly high curvature region and remains there. The gradient variance also remains high. 
As our optimization procedure can handle relatively high variance and curvature, it proceeds to optimize within this sharp but potentally non-local minimum.
On the other hand, it may be an advantage of fixed-step-size SGD that it skips over both high-variance and high-curvature minima. 

This failing of adaptive step sizes during the initial phase of training may be related to the ``short horizon bias''~\citep{wu2018understanding} of our one-step-ahead acquisition function. If so, compute budget can be used to approximate multi-step-ahead gains to help reduce this bias. Additionally, the ability to optimize within high-curvature high-variance regions could potentially be an advantage on problems with fewer local minima, yet this may not be the case for deep learning.

\section{Conclusion}
We introduced an online gradient estimation framework for stochastic gradient-based optimization, which leverages Hessian-vector products and variance estimates to perform automatic online gradient estimation and step size selection. 
The result is a stochastic optimization algorithm that can self-tune many important parameters such as momentum and learning rate schedules, in an online fashion without checkpointing or expensive outer-loop optimization.

While the required additional observables can be computed efficiently with recent advances in automatic differentiation tooling, they are of course not free, increasing computational cost and memory usage compared to \sgd{}. What one gains in return is automation, so that it suffices to run the algorithm just once, without tedious tuning.
Given the amount of human effort and computational resources currently invested into hyperparameter tuning, we believe our contributions are valuable steps towards fully-automated gradient-based optimization.


{
\bibliographystyle{plainnat}
\bibliography{filtering_sgd}

\begin{thebibliography}{44}
\providecommand{\natexlab}[1]{#1}
\providecommand{\url}[1]{\texttt{#1}}
\expandafter\ifx\csname urlstyle\endcsname\relax
  \providecommand{\doi}[1]{doi: #1}\else
  \providecommand{\doi}{doi: \begingroup \urlstyle{rm}\Url}\fi

\bibitem[Agarwal and Ganichev(2019)]{agarwal2019auto}
Ashish Agarwal and Igor Ganichev.
\newblock Auto-vectorizing {T}ensor{F}low graphs: {J}acobians, auto-batching
  and beyond.
\newblock \emph{arXiv preprint arXiv:1903.04243}, 2019.

\bibitem[Almeida et~al.(1999)Almeida, Langlois, Amaral, and
  Plakhov]{almeida1999parameter}
Lu{\'\i}s~B Almeida, Thibault Langlois, Jos{\'e}~D Amaral, and Alexander
  Plakhov.
\newblock Parameter adaptation in stochastic optimization.
\newblock In \emph{On-line learning in neural networks}, pages 111--134. 1999.

\bibitem[Amari(1998)]{amari1998natural}
Shun-Ichi Amari.
\newblock Natural gradient works efficiently in learning.
\newblock \emph{Neural computation}, 10\penalty0 (2):\penalty0 251--276, 1998.

\bibitem[Arnold et~al.(2019{\natexlab{a}})Arnold, Manzagol,
  Babanezhad~Harikandeh, Mitliagkas, and Le~Roux]{arnold2019transporting}
S\'{e}bastien Arnold, Pierre-Antoine Manzagol, Reza Babanezhad~Harikandeh,
  Ioannis Mitliagkas, and Nicolas Le~Roux.
\newblock Reducing the variance in online optimization by transporting past
  gradients.
\newblock In \emph{Advances in Neural Information Processing Systems 32}.
  2019{\natexlab{a}}.

\bibitem[Arnold et~al.(2019{\natexlab{b}})Arnold, Manzagol, Harikandeh,
  Mitliagkas, and Le~Roux]{arnold2019reducing}
S{\'e}bastien Arnold, Pierre-Antoine Manzagol, Reza~Babanezhad Harikandeh,
  Ioannis Mitliagkas, and Nicolas Le~Roux.
\newblock Reducing the variance in online optimization by transporting past
  gradients.
\newblock In \emph{Advances in Neural Information Processing Systems}, pages
  5391--5402, 2019{\natexlab{b}}.

\bibitem[Baydin et~al.(2017)Baydin, Cornish, Rubio, Schmidt, and
  Wood]{baydin2017online}
Atilim~Gunes Baydin, Robert Cornish, David~Martinez Rubio, Mark Schmidt, and
  Frank Wood.
\newblock Online learning rate adaptation with hypergradient descent.
\newblock \emph{arXiv preprint arXiv:1703.04782}, 2017.

\bibitem[Berrada et~al.(2019)Berrada, Zisserman, and
  Kumar]{berrada2019training}
Leonard Berrada, Andrew Zisserman, and M~Pawan Kumar.
\newblock Training neural networks for and by interpolation.
\newblock \emph{arXiv preprint arXiv:1906.05661}, 2019.

\bibitem[Bittner and Pronzato(2004)]{bittner2004kalman}
Barbara Bittner and Luc Pronzato.
\newblock Kalman filtering in stochastic gradient algorithms: construction of a
  stopping rule.
\newblock In \emph{2004 IEEE International Conference on Acoustics, Speech, and
  Signal Processing}, 2004.

\bibitem[Bottou et~al.(2018)Bottou, Curtis, and
  Nocedal]{bottou2018optimization}
L{\'e}on Bottou, Frank~E Curtis, and Jorge Nocedal.
\newblock Optimization methods for large-scale machine learning.
\newblock \emph{Siam Review}, 60\penalty0 (2):\penalty0 223--311, 2018.

\bibitem[Bradbury et~al.(2018)Bradbury, Frostig, Hawkins, Johnson, Leary,
  Maclaurin, and Wanderman-Milne]{jax2018github}
James Bradbury, Roy Frostig, Peter Hawkins, Matthew~James Johnson, Chris Leary,
  Dougal Maclaurin, and Skye Wanderman-Milne.
\newblock {JAX}: composable transformations of {P}ython+{N}um{P}y programs,
  2018.
\newblock URL \url{http://github.com/google/jax}.

\bibitem[Choi et~al.(2019)Choi, Shallue, Nado, Lee, Maddison, and
  Dahl]{choi2019empirical}
Dami Choi, Christopher~J Shallue, Zachary Nado, Jaehoon Lee, Chris~J Maddison,
  and George~E Dahl.
\newblock On empirical comparisons of optimizers for deep learning.
\newblock \emph{arXiv preprint arXiv:1910.05446}, 2019.

\bibitem[Dangel et~al.(2020)Dangel, Kunstner, and Hennig]{dangel2020backpack}
Felix Dangel, Frederik Kunstner, and Philipp Hennig.
\newblock Back{PACK}: Packing more into backprop.
\newblock In \emph{International Conference on Learning Representations}, 2020.

\bibitem[Duchi et~al.(2011)Duchi, Hazan, and Singer]{duchi2011adaptive}
John Duchi, Elad Hazan, and Yoram Singer.
\newblock Adaptive subgradient methods for online learning and stochastic
  optimization.
\newblock \emph{Journal of machine learning research}, 2011.

\bibitem[Eiben and Smit(2011)]{eiben2011parameter}
Agoston~E Eiben and Selmar~K Smit.
\newblock Parameter tuning for configuring and analyzing evolutionary
  algorithms.
\newblock \emph{Swarm and Evolutionary Computation}, 1\penalty0 (1):\penalty0
  19--31, 2011.

\bibitem[Gower et~al.(2017)Gower, Roux, and Bach]{gower2017tracking}
Robert~M Gower, Nicolas~Le Roux, and Francis Bach.
\newblock Tracking the gradients using the hessian: A new look at variance
  reducing stochastic methods.
\newblock \emph{arXiv preprint arXiv:1710.07462}, 2017.

\bibitem[He et~al.(2016{\natexlab{a}})He, Zhang, Ren, and Sun]{he2016deep}
Kaiming He, Xiangyu Zhang, Shaoqing Ren, and Jian Sun.
\newblock Deep residual learning for image recognition.
\newblock In \emph{Proceedings of the IEEE conference on computer vision and
  pattern recognition}, 2016{\natexlab{a}}.

\bibitem[He et~al.(2016{\natexlab{b}})He, Zhang, Ren, and Sun]{he2016identity}
Kaiming He, Xiangyu Zhang, Shaoqing Ren, and Jian Sun.
\newblock Identity mappings in deep residual networks.
\newblock In \emph{European conference on computer vision}, pages 630--645.
  Springer, 2016{\natexlab{b}}.

\bibitem[Johnson and Zhang(2013)]{johnson2013accelerating}
Rie Johnson and Tong Zhang.
\newblock Accelerating stochastic gradient descent using predictive variance
  reduction.
\newblock In \emph{Advances in neural information processing systems}, pages
  315--323, 2013.

\bibitem[Jones et~al.(1998)Jones, Schonlau, and Welch]{jones1998efficient}
Donald~R Jones, Matthias Schonlau, and William~J Welch.
\newblock Efficient global optimization of expensive black-box functions.
\newblock \emph{Journal of Global optimization}, 1998.

\bibitem[Kalman(1960)]{kalman1960}
Rudolph~Emil Kalman.
\newblock A new approach to linear filtering and prediction problems.
\newblock \emph{Transactions of the ASME--Journal of Basic Engineering},
  82\penalty0 (Series D):\penalty0 35--45, 1960.

\bibitem[Kingma and Ba(2014)]{kingma2014adam}
Diederik~P Kingma and Jimmy Ba.
\newblock {A}dam: A method for stochastic optimization.
\newblock \emph{arXiv preprint arXiv:1412.6980}, 2014.

\bibitem[Kushner(1964)]{kushner1964new}
Harold~J Kushner.
\newblock A new method of locating the maximum point of an arbitrary multipeak
  curve in the presence of noise.
\newblock 1964.

\bibitem[Mahsereci(2018)]{mahsereci2018probabilistic}
Maren Mahsereci.
\newblock \emph{Probabilistic Approaches to Stochastic Optimization}.
\newblock PhD thesis, Eberhard Karls Universit{\"a}t T{\"u}bingen T{\"u}bingen,
  2018.

\bibitem[Mahsereci and Hennig(2017)]{mahsereci2017probabilistic}
Maren Mahsereci and Philipp Hennig.
\newblock Probabilistic line searches for stochastic optimization.
\newblock \emph{The Journal of Machine Learning Research}, 2017.

\bibitem[Martens(2010)]{martens2010hfopt}
James Martens.
\newblock Deep learning via {H}essian-free optimization.
\newblock In \emph{Proceedings of the 27th International Conference on
  International Conference on Machine Learning}, 2010.

\bibitem[Martens(2014)]{martens2014new}
James Martens.
\newblock New insights and perspectives on the natural gradient method.
\newblock \emph{arXiv preprint arXiv:1412.1193}, 2014.

\bibitem[Nair and Hinton(2010)]{nair2010rectified}
Vinod Nair and Geoffrey~E Hinton.
\newblock Rectified linear units improve restricted boltzmann machines.
\newblock In \emph{Proceedings of the 27th international conference on machine
  learning (ICML-10)}, pages 807--814, 2010.

\bibitem[Patel(2016)]{Patel_2016}
Vivak Patel.
\newblock Kalman-based stochastic gradient method with stop condition and
  insensitivity to conditioning.
\newblock \emph{SIAM Journal on Optimization}, 2016.

\bibitem[Plagianakos et~al.(2001)Plagianakos, Magoulas, and
  Vrahatis]{plagianakos2001learning}
VP~Plagianakos, GD~Magoulas, and MN~Vrahatis.
\newblock Learning rate adaptation in stochastic gradient descent.
\newblock In \emph{Advances in convex analysis and global optimization}, pages
  433--444. Springer, 2001.

\bibitem[Polyak(1964)]{polyak1964some}
Boris~T Polyak.
\newblock Some methods of speeding up the convergence of iteration methods.
\newblock \emph{USSR Computational Mathematics and Mathematical Physics}, 1964.

\bibitem[Rolinek and Martius(2018)]{rolinek2018l4}
Michal Rolinek and Georg Martius.
\newblock L4: Practical loss-based stepsize adaptation for deep learning.
\newblock In \emph{Advances in Neural Information Processing Systems}, pages
  6433--6443, 2018.

\bibitem[S{\"a}rkk{\"a}(2013)]{sarkka2013bayesian}
Simo S{\"a}rkk{\"a}.
\newblock \emph{Bayesian filtering and smoothing}, volume~3.
\newblock Cambridge University Press, 2013.

\bibitem[Schmidt et~al.(2017)Schmidt, Le~Roux, and Bach]{schmidt2017minimizing}
Mark Schmidt, Nicolas Le~Roux, and Francis Bach.
\newblock Minimizing finite sums with the stochastic average gradient.
\newblock \emph{Mathematical Programming}, 162\penalty0 (1-2):\penalty0
  83--112, 2017.

\bibitem[Schneider et~al.(2019)Schneider, Balles, and
  Hennig]{schneider2019deepobs}
Frank Schneider, Lukas Balles, and Philipp Hennig.
\newblock Deepobs: A deep learning optimizer benchmark suite.
\newblock \emph{arXiv preprint arXiv:1903.05499}, 2019.

\bibitem[Schraudolph(1999)]{schraudolph1999local}
Nicol~N Schraudolph.
\newblock Local gain adaptation in stochastic gradient descent.
\newblock 1999.

\bibitem[{Shahriari} et~al.(2016){Shahriari}, {Swersky}, {Wang}, {Adams}, and
  {de Freitas}]{boreview}
B.~{Shahriari}, K.~{Swersky}, Z.~{Wang}, R.~P. {Adams}, and N.~{de Freitas}.
\newblock Taking the human out of the loop: A review of bayesian optimization.
\newblock \emph{Proceedings of the IEEE}, 2016.

\bibitem[Vaswani et~al.(2019)Vaswani, Mishkin, Laradji, Schmidt, Gidel, and
  Lacoste-Julien]{vaswani2019painless}
Sharan Vaswani, Aaron Mishkin, Issam Laradji, Mark Schmidt, Gauthier Gidel, and
  Simon Lacoste-Julien.
\newblock Painless stochastic gradient: Interpolation, line-search, and
  convergence rates.
\newblock In \emph{Advances in Neural Information Processing Systems 32}. 2019.

\bibitem[Vuckovic(2018)]{vuckovic2018kalman}
James Vuckovic.
\newblock Kalman gradient descent: Adaptive variance reduction in stochastic
  optimization, 2018.

\bibitem[Wu et~al.(2018)Wu, Ren, Liao, and Grosse]{wu2018understanding}
Yuhuai Wu, Mengye Ren, Renjie Liao, and Roger Grosse.
\newblock Understanding short-horizon bias in stochastic meta-optimization.
\newblock \emph{arXiv preprint arXiv:1803.02021}, 2018.

\bibitem[Wu and He(2018)]{wu2018group}
Yuxin Wu and Kaiming He.
\newblock Group normalization.
\newblock In \emph{Proceedings of the European Conference on Computer Vision
  (ECCV)}, 2018.

\bibitem[Yang et~al.(2013)Yang, Deb, Loomes, and
  Karamanoglu]{yang2013framework}
Xin-She Yang, Suash Deb, Martin Loomes, and Mehmet Karamanoglu.
\newblock A framework for self-tuning optimization algorithm.
\newblock \emph{Neural Computing and Applications}, 23\penalty0 (7-8):\penalty0
  2051--2057, 2013.

\bibitem[Ypma(1995)]{ypma1995historical}
Tjalling~J Ypma.
\newblock Historical development of the newton--raphson method.
\newblock \emph{SIAM review}, 37\penalty0 (4):\penalty0 531--551, 1995.

\bibitem[Yu et~al.(2006)Yu, Aberdeen, and Schraudolph]{yu2006fast}
Jin Yu, Douglas Aberdeen, and Nicol~N Schraudolph.
\newblock Fast online policy gradient learning with smd gain vector adaptation.
\newblock In \emph{Advances in neural information processing systems}, pages
  1185--1192, 2006.

\bibitem[Zeiler(2012)]{zeiler2012adadelta}
Matthew~D Zeiler.
\newblock Adadelta: an adaptive learning rate method.
\newblock \emph{arXiv preprint arXiv:1212.5701}, 2012.

\end{thebibliography}
}

\clearpage

\appendix

\section{The Full \Meka{} Algorithm with Adaptive Step Sizes}
\label{app:psuedocode}

\label{alg:meka}
\begin{algorithm}
  \caption{\Meka{} with adaptive step sizes based on maximizing the probability of improvement.}
  \label{alg:meka_opt}
\begin{algorithmic}
    \STATE {\bfseries Hyperparameters:} decay rates $\beta_r$=0.999, $\beta_\Sigma$=0.999, $\beta_\alpha$=0.999
    \STATE $m_0, P_0 \gets \vec{0}, 10^4$ 
    \COMMENT{large variance initialization ensures first Kalman gain is one}
    \STATE $u_0, s_0 \gets 0, 10^4$ 
    \STATE $\delta_0 = \vec{0}$
    \STATE $t \gets 0$
    \REPEAT
    \STATE $t \gets t + 1$
    \STATE $f_t^{(i)}, \nabla f_t^{(i)}, \nabla^2 f_t^{(i)}\delta_{t-1} \gets \text{VectorizedMap}(f, \{x_i\}_{i=1}^M; \theta_{t-1})$ 
    \COMMENT{compute per-example quantities}
    \STATE $y_t, r_t\; \gets \textnormal{MeanVarEMA}(\{ f_t^{(i)} \}; \beta_r)$ 
    \COMMENT{exponential moving average (EMA) on the variances}
    \STATE $g_t, \Sigma_t \gets \textnormal{MeanVarEMA}(\{ \nabla f_t^{(i)} \}; \beta_\Sigma)$
    \STATE $b_t, Q_t \gets \textnormal{MeanVar}(\{ \nabla^2 f_t^{(i)}\delta_{t-1} \})$
    \STATE $u_t, s_t \gets \textnormal{FilterUpdate}(u_{t-1}, s_{t-1}; m_{t-1}, P_{t-1}, y_t, r_t, b_t, Q_t)$
    \COMMENT{filter update equations}
    \STATE $m_t, P_t \gets \textnormal{FilterUpdate}(m_{t-1}, P_{t-1}; g_t, \Sigma_t, b_t, Q_t)$
    \COMMENT{filter update equations}
    \STATE $\alpha_t \gets \argmin_\alpha \eqref{eq:pi_loss}$  
    \COMMENT{with an EMA (decay rate $\beta_\alpha$) on the coefficients}
    \STATE $\delta_t \gets \alpha_t m_t$
    \STATE $\theta_t \gets \theta_{t-1} - \delta_t$
    \UNTIL{convergence}
\end{algorithmic}
\end{algorithm}

\section{The Function Value Dynamics Model}\label{app:func}

We discuss inferring the function value $f_t$, taking into account uncertainty due to changes in function value and observation noise during optimization.
The gradient dynamics~\eqref{eq:grad_posterior} imply the following dynamics model for the function value itself:
\begin{equation}
\begin{split}
    f_t \mid f_{t-1} &\sim \mathcal{N}(f_{t-1} + m_{t-1}^T\delta_{t-1} + \frac{1}{2}\delta_{t-1}^TB_t\delta_{t-1}, \\
    &\quad\quad \lambda_t + \delta_{t-1}^TP_{t-1}\delta_{t-1} + \frac{1}{4}\delta_{t-1}^TQ_t\delta_{t-1}) \\
    y_t \mid f_t &\sim \mathcal{N}(f_t, r_t) \\
\end{split}
\end{equation}
where we again use a quadratic approximation using Taylor expansion. Instead of the intractable $\nabla f$ and $\nabla^2 f$, we use the estimates from Section~\ref{sec:gradfilter}. The observations $y_t$ and $r_t$ are the empirical mean and variance of $f_t$ from a minibatch. The variance terms in the dynamics model are due to the uncertainty associated with $m_{t-1}$ and $B_{t}\delta_{t-1}$.

Here we have included a scalar term $\lambda_t$, which acts as a correction to the local quadratic model. This acts similar to a damping component, except we can automatically infer an optimal $\lambda_t$ by maximizing the likelihood of $p(y_t \mid y_{1:{t-1}})$, with a closed form solution~(see Appendix~\ref{app:lambda_t}). While damping terms are usually difficult to set empirically~\citep{choi2019empirical}, we note that including our $\lambda_t$ term is essentially free, and it automatically decays after optimization stabilizes.

\subsection{Adaptively Correcting the Dynamics Model}
\label{app:lambda_t}

We construct a dynamics model of the function value as follows (repeated for convenience):
\begin{equation}
\begin{split}
    f_t \mid f_{t-1} &\sim \mathcal{N}(f_{t-1} + m_{t-1}^T\delta_{t-1} + \frac{1}{2}\delta_{t-1}^TB_t\delta_{t-1},\; \lambda_t + \delta_{t-1}^TP_{t-1}\delta_{t-1} + \frac{1}{4}\delta_{t-1}^TQ_t\delta_{t-1}) \\
    y_t \mid f_t &\sim \mathcal{N}(f_t, r_t) \\
\end{split}
\end{equation}
We include a scalar parameter $\lambda_t$ in case the local quadratic approximation is inaccurate, ie. when $y_t$ is significantly different from the predicted value. If this occurs, a high value of $\lambda_t$ causes the Kalman gain to become large, throwing away the stale estimate and putting more weight on the new observed function value. 

We pick a value for $\lambda_t$ by maximizing the likelihood of $p(y_t | y_{1:t-1})$. Marginalizing over $f_t$, we get
\begin{equation}
\begin{split}
p(y_{t} | y_{1:t-1}) &= \mathcal{N}\left( \underbrace{u_{t-1} + m_{t-1}^T\delta_{t-1} + \frac{1}{2}\delta_{t-1}^TB_t\delta_{t-1}}_{u_t^-},\; \lambda_t + \underbrace{s_{t-1} +  \delta_{t-1}^TP_{t-1}\delta_{t-1} + \frac{1}{4}\delta_{t-1}^TQ_t\delta_{t-1} + r_t}_{c_t} \right) \\
\end{split}
\end{equation}
Taking the log and writing it out, we get
\begin{equation}
    \log p(y_{t} | y_{1:t-1}) \propto -\frac{1}{2} \left[ \frac{(y_t - u_t^-)^2}{\lambda_t + c_t} + \log (\lambda_t + c_t)\right]
\end{equation}
and its derivative is
\begin{equation}
    -\frac{1}{2}\left[ \frac{-(y_t - u_t^-)^2}{(\lambda_t + c_t)^2} + \frac{1}{\lambda_t + c_t} \right] = -\frac{1}{2}\left[ \frac{-(y_t - u_t^-)^2 + \lambda_t + c_t}{(\lambda_t + c_t)^2} \right]
\end{equation}
Setting this to zero, we get
\begin{equation}
\lambda_t = (y_t - u_t^-)^2 - c_t
\end{equation}
Since the role of $\lambda_t$ is to ensure we are not overconfident in our predictions, and we don't want to deal with negative variance values, we set
\begin{equation}
    \lambda_t^* = \max\{ (y_t - u_t^-)^2 - c_t,\; 0 \}
\end{equation}

As can be seen from Figure~\ref{fig:lambda_vs_s}, this $\lambda_t$ term goes to zero when it is not needed, ie. when the dynamics model is correct, which occurs on MNIST after convergence. It is also a quantity that shows us just how incorrect our dynamics model is, and for the problems we tested, we find that it is significantly smaller than the posterior variance $s_t$. This suggests that it has minimal impact if removed, but we keep it in the algorithm for cases when a quadratic approximation is not sufficient.

\begin{figure}[H]
    \centering
    \includegraphics[width=0.9\linewidth]{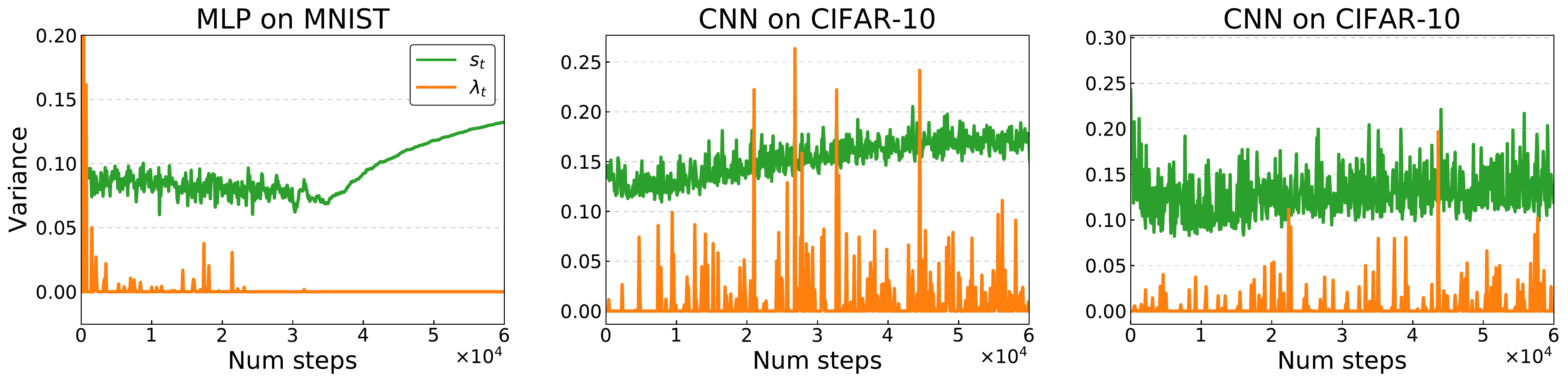}
    \caption{The quantity $2s_t + \lambda_t$ shows up as a constant variance term during step size adaptation. We find that though $\lambda_t$ has an effect only during a few iterations, it is usually small enough to be ignored. This suggests that the quadratic approximation assumption is okay most of the time. Nevertheless, a self-correcting term that is essentially compute-free is a desirable component.}
    \label{fig:lambda_vs_s}
\end{figure}

\subsection{Dealing with negative curvature in step size adaptation}
\label{app:negcurv_lambda}

When the gradient estimating is pointing in a direction of negative curvature, there is a chance that the optimal step size is infinity (Figure~\ref{fig:negcurv_nolambda}). 

\begin{figure}[H]
    \centering
    \begin{subfigure}[b]{0.35\linewidth}
        \includegraphics[width=\linewidth]{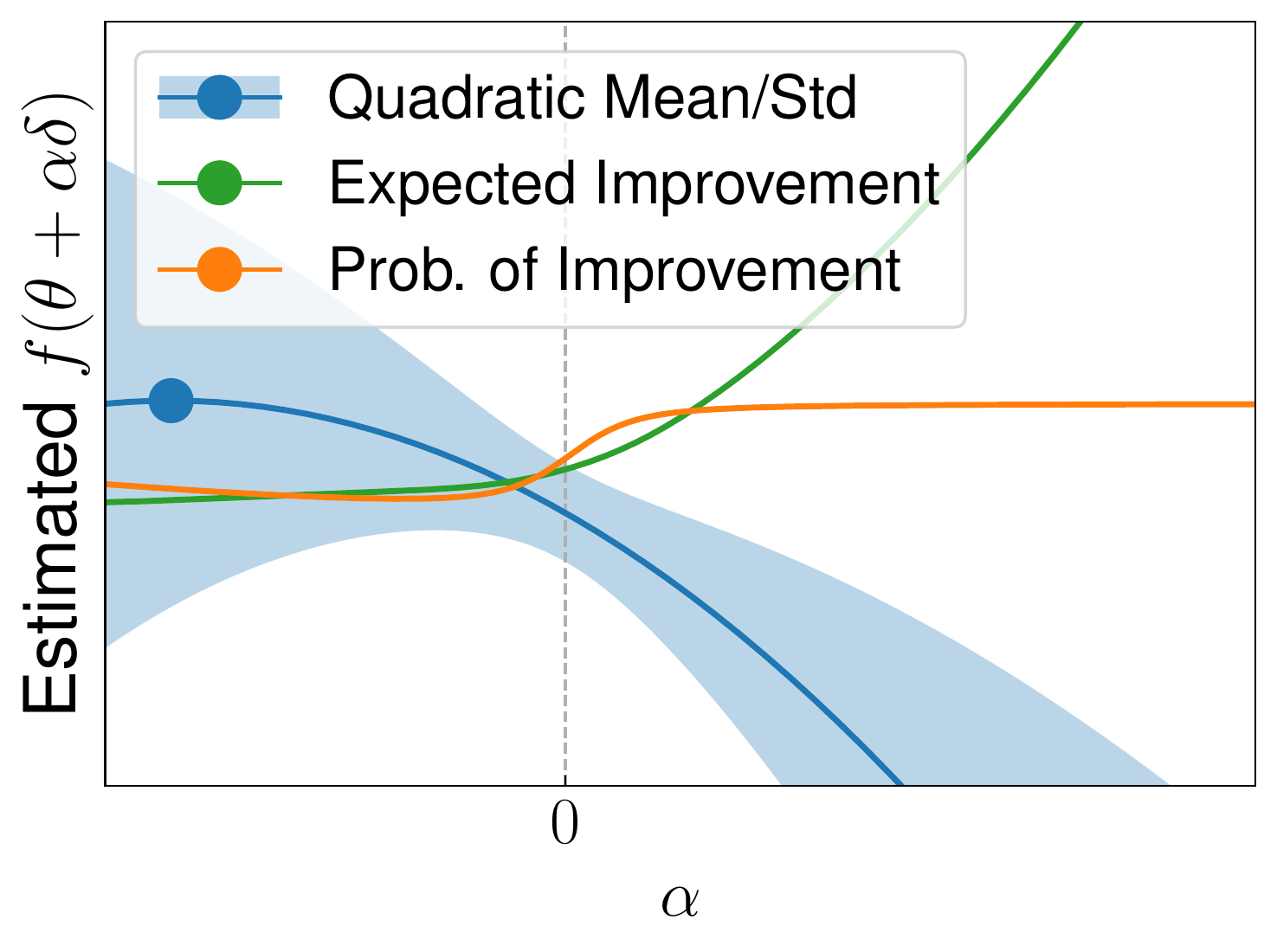}
        \caption{Negative curvature.}
        \label{fig:negcurv_nolambda}
    \end{subfigure}
    \begin{subfigure}[b]{0.35\linewidth}
        \includegraphics[width=\linewidth]{plots/negative_curv2_lambda.pdf}
        \caption{Negative curvature with $\lambda_t||\delta||^6$.}
        \label{fig:negcurv_lambda}
    \end{subfigure}
    \caption{Negative curvature can result in infinite step sizes. An extra correction term to the variance ensures the optimal step size is finite.}
    \label{fig:my_label}
\end{figure}

This occurs when the variance of $f_{t+1} - f_t$ rises slower than the expectation. To handle this situation, we can add a third order correction term, which appears in the variance as a term that scales with $||\delta||^6$. Using the same procedure as inferring a constant $\lambda_t$, we can instead add the term $\lambda_t||\delta||^6$ to the variance. We then choose $\lambda_t$ as
\begin{equation}
    \lambda_t^* = \max\left\{ \frac{1}{||\delta||^6} \left( (y_t - u_t^-)^2 - c_t \right),\; 0 \right\}
\end{equation}
This extra term (if $\lambda_t > 0$) in the variance ensures that variance increases faster than the expectation. Adaptive step sizes based on the probability of improvement will then have an optimal step size that is finite in value (Figure~\ref{fig:negcurv_lambda}). An additional damping effect may be added by lower bounding $\lambda_t$. We did not fully test this approach as the exponential moving averaged curvature used in practice was always positive for our test problems. Incidentally, Figures~\ref{fig:negcurv_nolambda} and~\ref{fig:negcurv_lambda} show that the expected improvement is not a good heuristic as it is extremely large even with this extra term. Moreover, the quadratic approximation will result in negative step sizes.

\section{Using the Current Hessian vs the Previous Hessian}
\label{app:prevhess}

For the dynamics model, we make the choice to use the Hessian at the updated location $B_t$, which is an unbiased estimate of $\nabla^2 f(\theta_{t})$, instead of $B_{t-1}$, the Hessian at $\theta_{t-1}$. Firstly, we made this choice for computational reasons: it is easier to compute the Hessian at $\theta_t$ since we are already computing the gradient $g_t$ evaluated at $\theta_t$. Secondly, we found that using the current Hessian results in better performance and more stability. Figure~\ref{fig:prevhess} shows this on MNIST. For CIFAR-10, we found that using the previous Hessian $B_{t-1}$ resulted in immediate divergence and NaNs, so we do not show those plots.

\begin{figure}[H]
    \centering
    \includegraphics[width=0.45\linewidth]{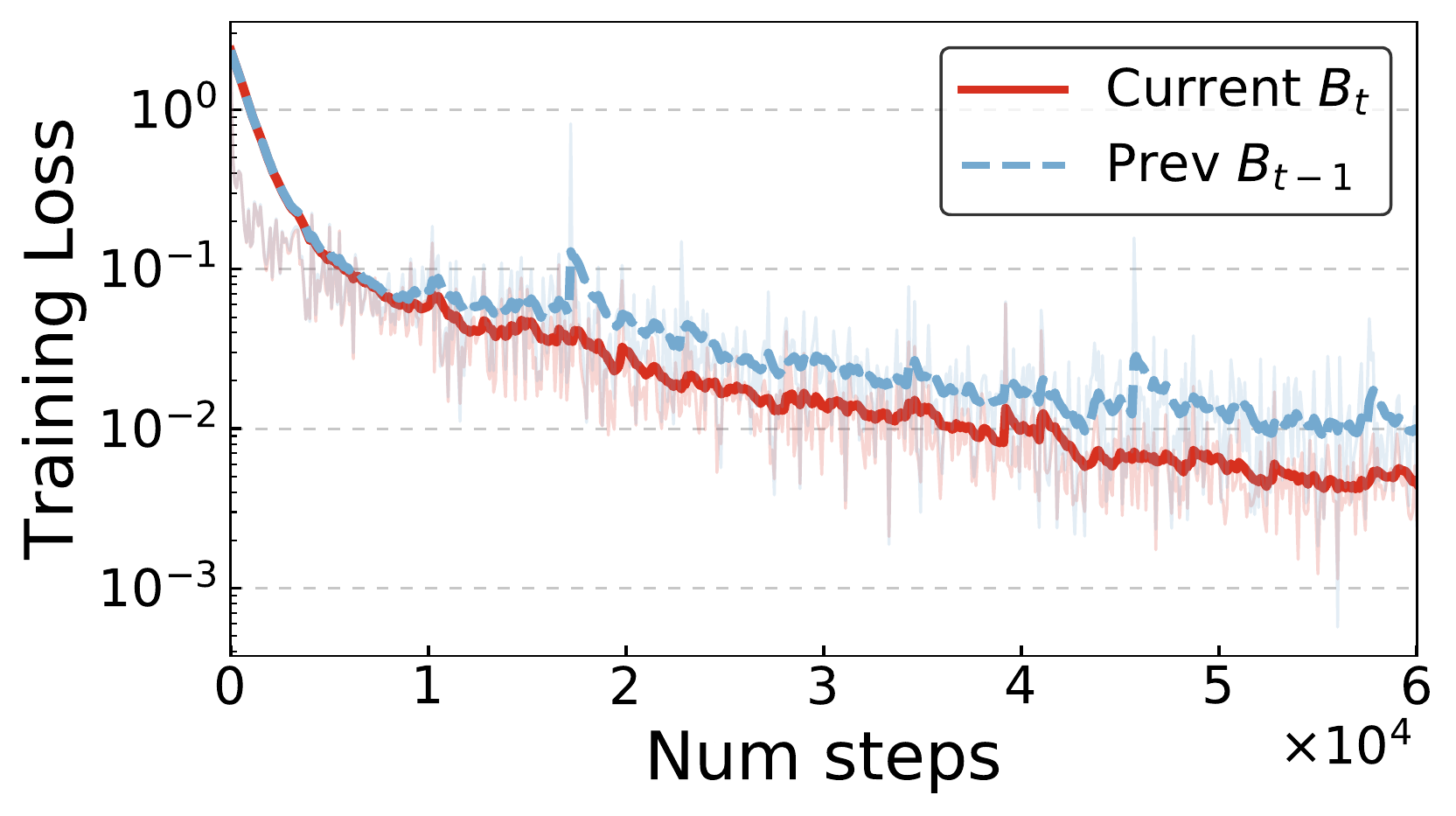}
    \caption{The choice of using current vs previous Hessian on MNIST.}
    \label{fig:prevhess}
\end{figure}

\section{Comparison of Adaptive Step Size Schemes}
\label{app:adaptive_comparison_full}

\begin{figure}[H]
    \centering
    \includegraphics[width=\linewidth]{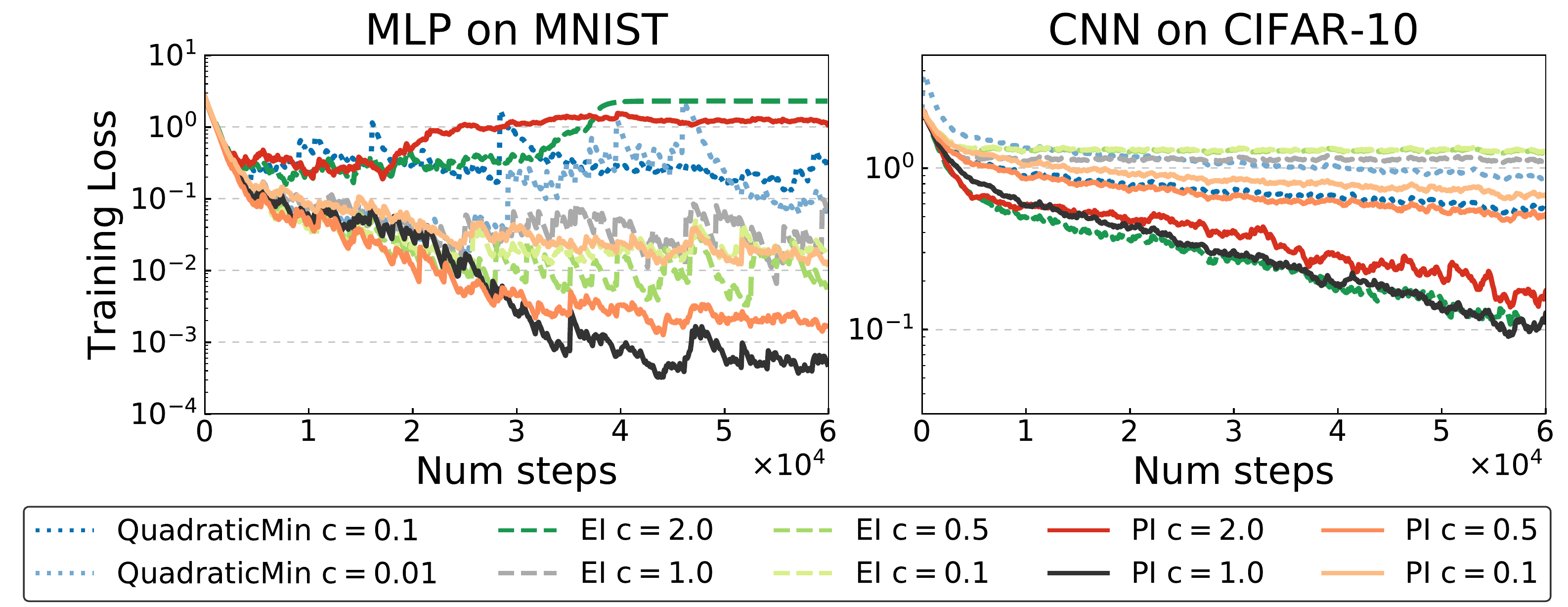}
    \caption{Comparing step sizes with a scaling factor $c$ such that the update is $\theta_t = \theta_{t-1} + c \alpha_t \delta_t$. This comparison includes expected improvement (EI). We note that it performs poorly on MNIST and requires a scaling of 2.0 to match PI on CIFAR-10. Using a smaller scaling factor results in worse performance.}
    \label{fig:adaptive_comparison_full}
\end{figure}

\subsection{Comparison with Tuned Optimizers} \label{app:comparison_with_all}

\begin{figure}
    \includegraphics[width=\linewidth]{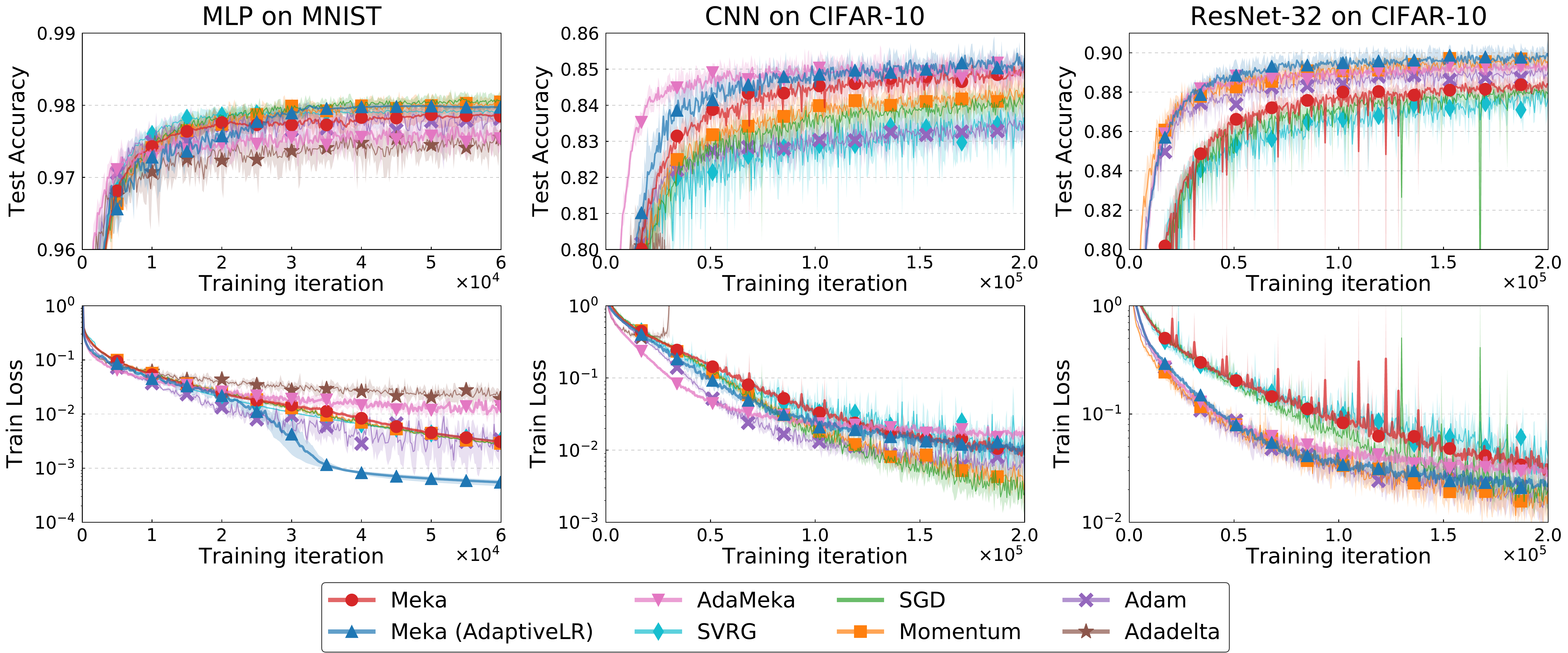}
    \caption{\Meka{} is competitive with optimizers that have additional tunable hyperparameters. 
    Results are averaged over 5 random seeds; shaded regions are 5th and 95th percentiles.}
    \label{fig:all_workloads}
\end{figure}

We measure the performance of our method against a variety of other approaches. We compare with fixed step size versions of \sgd{}, \sgd{} with momentum~\citep{polyak1964some}, and \adam{}~\citep{kingma2014adam}. For these, we tune the step size using grid search. We also compare against \adadelta{}~\citep{zeiler2012adadelta}, which is a competing learning rate-free algorithm, and \svrg{}~\citep{johnson2013accelerating}, an optimization algorithm focused around estimating the full batch gradient, using the same step size as the tuned \sgd. For comparison, we implemented \Meka{} using the same constant learning rate as the tuned \sgd, and \adameka{} with the same learning rate as the tuned \adam. We applied fully adaptive step sizes with \Meka{} update directions; on ResNet-32, we used \adameka{} update directions to mitigate poor conditioning.

Figure~\ref{fig:all_workloads} shows the resulting loss and accuracy curves. As this includes optimizers across a wide range of motivations, we highlight some specific comparisons. \Meka{} generally performs better than \sgd{} in terms of test accuracy, showing that too much stochasticity can hurt generalization. Though both designed with gradient estimation in mind, \Meka{} seems to compare favorable against \svrg{} in terms of performance. We note that the per-iteration costs of \Meka{} were also cheaper as \svrg{} requires two gradient evaluations. With the default learning rate of $1.0$, \adadelta{} performs decently on MNIST but ends up diverging on CIFAR-10. In comparison, our adaptive step sizes converge well and generally outperform fixed step sized \Meka{} (or \adameka{}) in both trainng loss and test accuracy.

\section{Experiment Details}
\label{app:experiment_details}
\subsection{Dataset Description}
We used the official train and test split for MNIST and CIFAR-10. We did not do any data augmentation for MNIST. For CIFAR-10, we normalized the images by subtracting every pixel with the global mean and standard deviation across the training set. In addition to this, unless specified otherwise, we pre-processed the images following \citet{he2016deep} by padding the images by 4 pixels on each side and applying random cropping and horizontal flips.
\subsection{Architecture Description}
All models use the the ReLU~\citep{nair2010rectified} activation function.
\paragraph{MLP}
We used an MLP with 1 hidden layer of 100 hidden units.
\paragraph{CNN} 
We used the same architecture as the ``3c3d'' architecture in \citet{schneider2019deepobs}, which consists of 3 convolutional layers with max pooling, followed by 3 fully connected layers. The first convolutional layer has a kernel size of $5 \times 5$ with stride 1, ``valid'' padding, and 64 filters. The second convolutional layer has a kernel size of $3 \times 3$ with stride 1, ``valid'' padding, and 96 filters. The third convolutional layer has a kernel size of $3 \times 3$ with stride 1, ``same'' padding, and 128 filters. The max pooling layers have a window size of $3 \times 3$ with stride 2. The 2 fully connected layers have 512 and 256 units respectively.
\paragraph{ResNet-32}
Our ResNet-32~\citep{he2016deep} model uses residual blocks based on \citet{he2016identity}. We replaced the batch normalization layers with group normalization~\citep{wu2018group} as batch-dependent transformations conflict with our assumption that the gradient samples are independent, and hinder our method to estimate gradient variance.

\subsection{Optimizer Comparisons Description}
We tuned the step size of \sgd{} in a grid of $\{0.001, 0.01, 0.1, 1.0\}$. We tuned the step size of \sgd{} with momentum, and \adam{} in a grid of $\{0.0001, 0.001, 0.01, 0.1\}$. We chose the best step size of \sgd{} for the variant of \Meka{} with a constant learning rate, and for \svrg. 

The chosen step size for \sgd{} was 0.1 for MNIST, and 0.1 for CIFAR-10. For \sgd{} with momentum, the chosen step size was 0.01 for MNIST, 0.1 for ResNet-32 on CIFAR-10, and 0.001 for CNN on CIFAR-10. The best step size for \adam{} was 0.001 for MNIST and ResNet-32 on CIFAR-10, and 0.0001 for CNN on CIFAR-10.

For \sgd{} with momentum, the momentum coefficient $\gamma$ was fixed to 0.9. For \adam, $\beta_1$, $\beta_2$, $\varepsilon$ were fixed to 0.9, 0.999, and $10^{-8}$ respectively. For \adadelta, $\rho$ and $\varepsilon$ were fixed to 0.95 and $10^{-6}$ respectively. 

\section{Sensitivity of \Meka's Hyperparameters}
\label{app:beta_test}

\begin{figure}[!h]
    \centering
    \includegraphics[width=\linewidth]{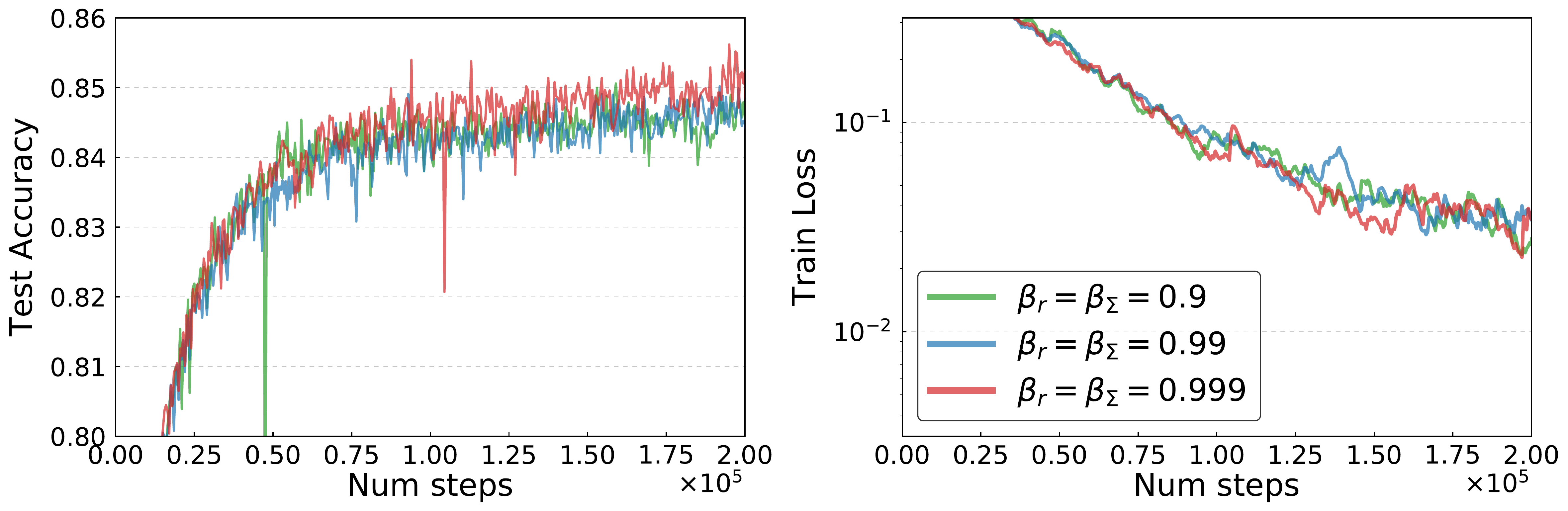}
    \vspace{-1.7em}
    \caption{The performance of \Meka{} with constant learning rate for CNN on CIFAR-10 is not sensitive to the choice of the exponential moving average decay rates $\beta_r$ and $\beta_\Sigma$.}
\end{figure}

\begin{figure}[!h]
    \centering
    \includegraphics[width=\linewidth]{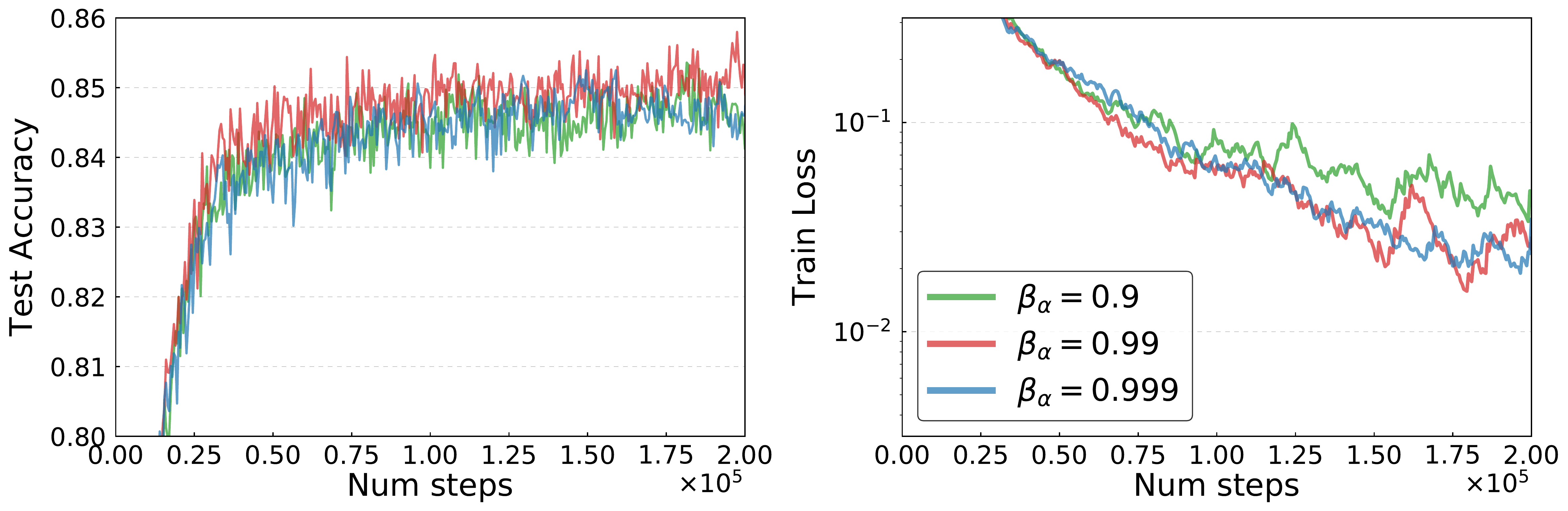}
    \vspace{-1.7em}
    \caption{The performance of \Meka{} with the PI adaptive scheme for CNN on CIFAR-10 is not sensitive to the choice of the exponential moving average decay rate $\beta_\alpha$.}
\end{figure}

\section{Additional Cost of \Meka}
\label{app:mem_compute_cost}

\begin{table}[H]
\caption{The memory cost of using a vectorized map to obtain individual gradients is greater than taking the gradient of a sum over the minibatch, whereas the asymptotic compute cost is the same. $B$ is the batch size, $|\theta|$ is the number of parameters, and $D$ is the number of activations in the model.}
\label{tab:mem_and_compute}
\centering
\begin{tabular}{@{}lcc@{}}
\toprule
 & $\nabla \sum_{i=1}^m f_t^{(i)}$ & $\text{VMap}(\nabla, f_t^{(i)})$ \\ \midrule
Memory & $\mathcal{O}\left(|\theta| + BD\right)$ & $\mathcal{O}\left(B(|\theta| + D)\right)$ \\
Compute & $\mathcal{O}\left(BD|\theta|\right)$ & $\mathcal{O}\left(BD|\theta|\right)$ \\ \bottomrule
\end{tabular}
\end{table}

\begin{table}[H]
\caption{The ratio of the time it takes to complete one iteration for \Meka{} versus \sgd{}. Note that in addition to the vector mapped gradients, we also compute an additional Hessian-vector product. The runtimes are after just-in-time compilation of JAX has settled. Runtimes are tested on the NVIDIA TITAN Xp GPU.}
\label{tab:runtime}
\centering
\begin{tabular}{@{}lllrr@{}}
\toprule
Dataset & Architecture & SGD & \multicolumn{1}{l}{Meka (fixed lr)} & \multicolumn{1}{l}{Meka (PI adaptive)} \\ \midrule
MNIST & MLP & 1.00 & 1.10 & 1.00 \\
CIFAR-10 & CNN & 1.00 & 0.83 & 1.34 \\
CIFAR-10 & ResNet-32 & 1.00 & 1.68 & 2.97 \\ \bottomrule
\end{tabular}
\end{table}

\section{Proofs}

\begin{proof}[Proof of Proposition~\ref{proposition:quadratic_toy_problem_convergence}]
A standard Lipschitz bound yields
\begin{equation}
    \begin{split}
        \E[f_{t+1}] & \leq \E [f_t] - \alpha \E[ \nabla f_t^T m_t] + \frac{L\alpha^2}{2} \E [\Vert m_t\Vert^2] \\
        & \leq \E[f_t] - \frac{\alpha}{2}\left(\E[2 \nabla f_t^Tm_t - \Vert m_t\Vert^2]  \right) \\
        & = \E[f_t] - \frac{\alpha}{2} \left( \E[ \Vert \nabla f_t\Vert^2] - \E[\Vert m_t - \nabla f_t\Vert^2]  \right).
    \end{split}
\end{equation}
Using strong convexity ($\Vert \nabla f_t\Vert^2 \geq 2\mu (f_t - f_\ast)$) and subtracting $f_\ast$ from both sides results in
\begin{equation}
    \label{eq:convergence_linear_plus_variance}
    \E[f_{t+1} - f_\ast] \leq (1 - \alpha\mu) \E[f_t - f_\ast] + \frac{\alpha}{2}\E[\Vert m_t - \nabla f_t\Vert^2]
\end{equation}
So in each step, we get a multiplicative decrease in the expected function value (left term) but we add a term that depends on the variance of our filtered gradient estimate $m_t$.
So, in essence, to establish convergence, we have to show that $\E[\Vert m_t - \nabla f_t\Vert^2]$ decreases to zero sufficiently fast.

Since all assumptions of the Kalman filter are satisfied, we know that $\E[m_t] = \E[\nabla f_t]$ and $\E[(m_t - \nabla f_t)(m_t - \nabla f_t)^T] = P_t$.
Hence, $\E[\Vert m_t - \nabla f_t\Vert^2] = \trace(P_t)$.
We now show inductively that $P_t = \frac{1}{t+1}\Sigma$.
This holds for $t=0$ by construction.
Assume it holds for arbitrary but fixed $t-1$.
Then
\begin{equation}
    K_t = P_{t-1}(P_{t-1} + \Sigma)^{-1} = \frac{1}{t} \Sigma \left( \frac{1}{t}\Sigma + \Sigma \right)^{-1} = \frac{1}{t} \Sigma \left( \frac{t+1}{t} \Sigma \right) ^{-1} = \frac{1}{t+1} I
\end{equation}
and, thus,
\begin{equation}
    P_t = (I - K_t)P_{t-1} = \left( I - \frac{1}{t+1} I \right) \frac{1}{t} \Sigma = \frac{1}{t+1} \Sigma
\end{equation}

Plugging $\E[\Vert m_t - \nabla f_t\Vert^2] = \trace(P_t) = \frac{1}{t+1} \trace(\Sigma)$ back into Eq.~\eqref{eq:convergence_linear_plus_variance} and introducing the shorthands $e_t = \E[f_t -f_\ast]$ and $\sigma^2 := \trace(\Sigma)$ reads
\begin{equation}
    e_t \leq (1-\alpha\mu)e_{t-1} + \frac{\alpha\sigma^2}{2} \frac{1}{t}.
\end{equation}
Iterating backwards results in
\begin{equation}
    e_t \leq (1-\alpha\mu)^t e_0 + \frac{\alpha \sigma^2}{2} \sum_{s=0}^{t-1} \frac{(1-\alpha\mu)^{t-1-s}}{s+1} = (1-\alpha\mu)^t e_0 + \frac{\alpha \sigma^2}{2} \sum_{s=1}^t \frac{(1-\alpha\mu)^{t-s}}{s}.
\end{equation}
Lemma~\ref{lemma:geometric_times_harmonic_sequence} shows that the sum term is $O(1/t)$. The first (exponential) term is trivially $O(1/t)$, which concludes the proof.
\end{proof}

\begin{lemma}
\label{lemma:geometric_times_harmonic_sequence}
Let $0< c < 1$ and define the sequence (for $t\geq 1$)
\begin{equation*}
    a_t = \sum_{s=1}^t \frac{c^{t-s}}{s}.
\end{equation*}
Then $a_t\in O(\frac{1}{t})$.
\end{lemma}
\begin{proof}
Let $T$ be the smallest index such that $c\frac{T+1}{T} < 1$, i.e., $T=\lceil c/ (1-c) \rceil$.
Define
\begin{equation}
    M = \max\left( T a_T, \left(1- c\frac{T+1}{T}\right)^{-1} \right)
\end{equation}
This ensures that $a_T \leq \frac{M}{T}$ and
\begin{equation}
    \label{eq:bound_for_M}
    \frac{1}{M} + c\frac{t+1}{t} \leq 1
\end{equation}
for all $t\geq T$.
We now show inductively that $a_t \leq \frac{M}{t}$ for all $t\geq T$.
It holds for $t=T$ by construction of $M$.
Assume it holds for some $t\geq T$.
Then
\begin{equation}
    \begin{split}
        a_{t+1} & = \sum_{s=1}^{t+1} \frac{c^{t+1-s}}{s} = \frac{1}{t+1} + c \underbrace{ \sum_{s=1}^t \frac{c^{t-s}}{s} }_{=a_t \leq M/t} \\
        & \leq \frac{1}{t+1} + c\frac{M}{t} = \frac{M}{t+1} \underbrace{ \left( \frac{1}{M} + c\frac{t+1}{t} \right) }_{\leq 1 \text{ by Eq.~\eqref{eq:bound_for_M}}} \leq \frac{M}{t+1}.
    \end{split}
\end{equation}
\end{proof}

\end{document}